\UseRawInputEncoding
\documentclass[10pt,journal,compsoc]{IEEEtran}

\ifCLASSOPTIONcompsoc
  \usepackage[nocompress]{cite}
\else
  \usepackage{cite}
\fi

\ifCLASSINFOpdf
\else
\fi
\usepackage[numbers,sort&compress]{natbib}
\usepackage{graphicx}
\usepackage{amsmath,amsfonts,amssymb,amsthm}
\usepackage{algorithm}
\usepackage{algorithmic}
\usepackage{subfigure}
\renewcommand{\Re}{\mathbb{R}}
\newcommand{\Id}{\mathtt{I}}

\newcommand{\diag}[1]{\ensuremath{\mathrm{diag}\left(#1\right)}}

\newcommand{\rank}[1]{\ensuremath{\mathrm{rank}\left(#1\right)}}
\newcommand{\norm}[1]{\ensuremath{\left\| #1 \right\|}}

\newtheorem{prob}{Problem}[section]
\newtheorem{ques}{Question}
\newtheorem{theo}{Theorem}[section]
\newtheorem{defn}{Definition}[section]
\newtheorem{lemm}{Lemma}[section]
\newtheorem{coro}{Corollary}[section]

\begin{document}
\title{Recovery of Future Data via Convolution Nuclear Norm Minimization}

\author{Guangcan Liu,~\IEEEmembership{Senior Member,~IEEE}, Wayne Zhang
\IEEEcompsocitemizethanks{\IEEEcompsocthanksitem G. Liu is with the School of Automation, Southeast University, Nanjing, China 210018. Email: gcliu1982@gmail.com.\protect\\
\IEEEcompsocthanksitem W. Zhang is with SenseTime Research, Harbour View 1 Podium, 12 Science Park East Ave, Hong Kong Science Park,
       Shatin, N.T., Hong Kong 999077. Email: wayne.zhang@sensetime.com.}
\thanks{Copyright (c) 2017 IEEE. Personal use of this material is permitted. However, permission to use this material for any other purposes must be obtained from the IEEE by sending a request to pubs-permissions@ieee.org.}}

\IEEEtitleabstractindextext{
\begin{abstract}
This paper studies the problem of \emph{time series forecasting} (TSF) from the perspective of \emph{compressed sensing}. First of all, we convert TSF into a more inclusive problem called \emph{tensor completion with arbitrary sampling} (TCAS), which is to restore a tensor from a subset of its entries sampled in an \emph{arbitrary} manner. While it is known that, in the framework of Tucker low-rankness, it is theoretically impossible to identify the target tensor based on some arbitrarily selected entries, in this work we shall show that TCAS is indeed tackleable in the light of a new concept called \emph{convolutional low-rankness}, which is a generalization of the well-known \emph{Fourier sparsity}. Then we introduce a convex program termed Convolution Nuclear Norm Minimization (CNNM), and we prove that CNNM succeeds in solving TCAS as long as a sampling condition---which depends on the \emph{convolution rank} of the target tensor---is obeyed. This theory provides a meaningful answer to the fundamental question of what is the minimum sampling size needed for making a given number of forecasts. Experiments on univariate time series, images and videos show encouraging results.
\end{abstract}

}

\maketitle

\IEEEdisplaynontitleabstractindextext
\IEEEpeerreviewmaketitle

\section{Introduction}\label{sec:introduction}
\IEEEPARstart{C}{an} we predict the future? While seems mysterious to the general public, this question is of particular interest to a wide range of scientific areas, ranging from mathematics, physics and philosophy to finance, meteorology and engineering. \emph{Time series forecasting} (TSF)~\cite{suvey:2006}, which aims to predict future observations coming ahead of time based on historical data, plays an important role in the development of forecasting techniques. Formally, the problem can be described as follows:
\begin{prob}[Time Series Forecasting]\label{pb:mtsf}
Suppose that $\{\mathbf{M}_t\}_{t=1}^{p+h}$ is a sequence of $p+h$ order-$(n-1)$ tensors with $n\geq1$, i.e., $\mathbf{M}_t\in\Re^{m_1\times{}\cdots\times{}m_{n-1}}$. Given the historical part $\{\mathbf{M}_t\}_{t=1}^{p}$ consisting of $p$ observed samples, the goal is to predict the next $h$ unseen samples $\{\mathbf{M}_t\}_{t=p+1}^{p+h}$, where $h$ is often called the `forecast horizon'.
\end{prob}

Usually, the forecast horizon $h$ is provided by users and is determined according to the demands of planning or decision making. For example, in M4 Competition~\cite{M4:ijf:2020}, $h$ is a positive integer between 6 and 48. In general, the above problem has a wide scope that covers various settings of forecasting. When $n=1$ or $n=2$, Problem~\ref{pb:mtsf} corresponds to the classic \emph{univariate} or \emph{multivariate} TSF problems~\cite{suvey:2006}, respectively. In the case of $n=3$, $\mathbf{M}_t$ is of matrix-valued, thus Problem~\ref{pb:mtsf} embodies the \emph{video prediction} task investigated in~\cite{Cun2016}. Interestingly, as one will see, the proposed methods can handle the general case of $n\geq1$ in a universal way. Yet, it is worth mentioning that the most essential case is in fact $n=1$, i.e., univariate TSF, as acknowledged by many forecasting competitions~\cite{HYNDMAN20207,M4:ijf:2020}.

Despite its brief definition, the TSF problem is extremely difficult to crack and still demands new scenarios after decades of research~\cite{Pablo:arxiv:2020}. In particular, to our knowledge, the following fundamental question has not been previously answered:
\begin{ques}[Sampling Complexity]\label{ques:sc}
What is the minimum sampling size $p$ needed for predicting $h$ future samples, or, equivalently, what is the maximum forecast horizon $h$ when $p$ historical samples are available?
\end{ques}
Seeking a valid answer to the above question is of great practical significance, especially given the increasing demand for long-term forecasting (i.e., $h$ is large)~\cite{lfsf:2015:eaai}. To this end, we would like to try approaching the TSF problem from the perspective of \emph{compressed sensing}~\cite{Candes:2008:spm,candes:2005:tit,cs:2012:mark}. Namely, we convert Problem~\ref{pb:mtsf} into a more general problem as follows. For a tensor-valued time series $\{\mathbf{M}_t\}_{t=1}^{p+h}$ with $\mathbf{M}_t\in\Re^{m_1\times{}\cdots\times{}m_{n-1}}$, we define an order-$n$ tensor $\mathbf{L}_0\in\Re^{m_1\times{}\cdots\times{}m_{n}}$ as
\begin{align}
[\mathbf{L}_0]_{i_1,\cdots,i_n} = [\mathbf{M}_{i_n}]_{i_1,\cdots,i_{n-1}}, 1\leq{}i_j\leq{}m_j, 1\leq{}j\leq{}n,
\end{align}
where $m_n = p+h$ and $[\cdot]_{i_1,\cdots{},i_n}$ denotes the $(i_1,\cdots,i_n)$th entry of an order-$n$ tensor. In other words, $\mathbf{L}_0$ is formed by concatenating a sequence of order-$(n-1)$ tensors into an order-$n$ one. Subsequently, we define a sampling set $\Omega\subset\{1,\cdots,m_1\}\times\cdots\times\{1,\cdots,m_n\}$ as in the following:
\begin{align}\label{eq:construct:omega}
&\Omega = \{(i_1,\cdots,i_n): \\\nonumber
&\quad{}1\leq{}i_j\leq{}m_j, \forall{}1\leq{}j\leq{}n-1 \textrm{ and } 1\leq{}i_n\leq{}p\}.
\end{align}
That is, $\Omega$ is a set consisting of the locations of the observations in $\{\mathbf{M}_t\}_{t=1}^{p}$, while the complement set of $\Omega$, denoted as $\Omega^\bot$, stores the locations of the future entries in $\{\mathbf{M}_t\}_{t=p+1}^{p+h}$. With these notations, we turn to a more inclusive problem named \emph{tensor completion with arbitrary sampling} (TCAS).
\begin{prob}[Tensor Completion with Arbitrary Sampling]\label{pb:tc}
Let $\mathbf{L}_0\in\mathbb{R}^{m_1\times\cdots\times{}m_n}$ be the target tensor of interest. Suppose that we are given a subset of the entries in $\mathbf{L}_0$ and a sampling set
$\Omega\subset\{1,\cdots,m_1\}\times\cdots\times\{1,\cdots,m_n\}$ consisting of the locations of observed entries. The configuration of $\Omega$ is quite progressive, in a sense that the locations of the observed entries are allowed to be distributed \textbf{arbitrarily}. Can we restore the target tensor $\mathbf{L}_0$? If so, under which conditions?
\end{prob}
In this way, Problem~\ref{pb:mtsf} is incorporated into the scope of \emph{tensor completion}~\cite{akshay:2013:nips,Yuan2016OnTC,liuxy:2020:tit,wangjl:2021:tip}, which is to fill in the missing entries of a partially observed tensor. Such a scenario has an immediate advantage; that is, it becomes straightforward to handle the difficult cases where the historical part $\{\mathbf{M}_t\}_{t=1}^{p}$ itself is incomplete. This is because, regardless of whether the historical part is complete or not, the task always boils down to recovering $\mathbf{L}_0$ from $\mathcal{P}_{\Omega}(\mathbf{L}_0)$, where $\mathcal{P}_{\Omega}(\cdot)$ is the orthogonal projection onto the subspace of tensors supported on $\Omega$. Moreover, to obtain an accurate answer to Question~\ref{ques:sc}, there is no much loss to consider instead Problem~\ref{pb:tc}: The success condition for solving TCAS is dominated by the worst configuration of $\Omega$, and, coincidentally, the sampling pattern of forecasting is typically the worst case, as we will confirm in Section~\ref{sec:exp:simu}.

While appealing, the goal of TCAS seems ``unrealistic'' according to the existing tensor completion theories~\cite{Candes:2009:math,gross:tit:2011,tc:tpami:2013:ye,tc:tip:2018:zhou,lu:pami:2020}. Consider the case of $n=2$, i.e., matrix completion with arbitrary sampling. In this case, for the target matrix $\mathbf{L}_0$ to be recoverable from a subset of its entries, it has been shown in~\cite{liu:nips:2017,liu:tpami:2021} that the together of $\mathbf{L}_0$ and $\Omega$ ought to obey an \emph{isomeric condition}, which is obviously violated by the setup of arbitrary sampling---note that arbitrary sampling allows some columns and rows of $\mathbf{L}_0$ to be \emph{wholly missing}. Yet, we would like to clarify that the isomeric condition is indispensable only when the circumstance is restricted within the scope of Low-Rank Matrix Completion (LRMC)~\cite{Candes:2009:math}, where the target $\mathbf{L}_0$ has the lowest Tucker rank among all the matrices that satisfy the constraint given by the observed entries. Beyond the field of Tucker low-rankness, as we will show, it is indeed possible to solve TCAS without imposing any restrictions on the sampling pattern.

Interestingly, the tool for addressing TCAS had already been established several decades ago, namely the prominent \emph{Fourier sparsity}---the phenomenon that most of the Fourier coefficients of a signal are zero or approximately so. Concretely, TCAS can be addressed by a straightforward, but not specifically studied method termed Discrete Fourier Transform (DFT) based $\ell_1$ minimization ($\mathrm{DFT}_{\ell_1}$):
\begin{align}\label{eq:tc:dftl1:exact}
&\min_{\mathbf{L}}\norm{\mathcal{F}(\mathbf{L})}_1, \textrm{ s.t. }\mathcal{P}_{\Omega}(\mathbf{L}-\mathbf{L}_0)=0,
\end{align}
where $\mathcal{F}(\cdot)$ is the DFT operator, and $\|\cdot\|_1$ denotes the $\ell_1$ norm of a tensor seen as a long vector. Hereafter, denote by $\mathrm{card}(\cdot)$ and $\mathrm{vec}(\cdot)$ the cardinality of a set and the vectorization of a tensor, respectively. Let $\mathbf{z}=\mathrm{vec}(\mathcal{F}(\mathbf{L}))$ in~\eqref{eq:tc:dftl1:exact}. Then it is easy to see that $\mathrm{DFT}_{\ell_1}$ is a special case of the well-known $\ell_1$ minimization problem:
\begin{align}\label{eq:dftl1:equi}
\min_{\mathbf{z}\in\mathbb{C}^m} \|\mathbf{z}\|_1, \textrm{ s.t. } \Phi\mathbf{z}=\mathbf{y},
\end{align}
where $m=\Pi_{i=1}^nm_i$, $\Phi\in\mathbb{C}^{\tilde{p}\times{}m}$ is called the \emph{sensing matrix}, and $\mathbf{y}\in\mathbb{C}^{\tilde{p}}$ is a vector consisting of $\tilde{p}$ observations. The general case of~\eqref{eq:dftl1:equi} had been extensively studied in the literature, e.g.,~\cite{tit:2001:donoho,Candes:2008:spm,candes:2005:tit,cs:2012:mark,tit:2015:xu}. As for our $\mathrm{DFT}_{\ell_1}$, it is configured that $\tilde{p}=\mathrm{card}(\Omega)$, $\Phi$ is determined by the together of the \emph{Fourier transform matrices} (see Section~\ref{sec:dfttran}) and the sampling set $\Omega$, and $\mathbf{y}$ is composed of the observed entries in $\mathbf{L}_0$. In this particular case, whenever $n=1$ (i.e., $\mathbf{L}_0$ is a vector) and the observed entries are chosen randomly or specifically, as proven in~\cite{iss:2006:mark,iss:2010:Haupt}, the sensing matrix $\Phi$ satisfies the Restricted Isometry Property (RIP)~\cite{candes:2005:tit} and thus program~\eqref{eq:dftl1:equi} can be successful, implying that $\mathrm{DFT}_{\ell_1}$ can identify the target $\mathbf{L}_0$. These results, however, are inapplicable to TCAS, in which the sampling pattern is allowed to be \emph{arbitrary} and $n$ may be any natural number. Moreover, as we will show in Section~\ref{sec:optalg}, the optimization problem in~\eqref{eq:tc:dftl1:exact} can be efficiently solved within $\mathcal{O}(m\log{}m)$ time, whereas, by sharp contrast, $\mathcal{O}(m^3)$ time is usually needed for finding the optimal solution to~\eqref{eq:dftl1:equi}.
\begin{figure}[h!]
\begin{center}
\includegraphics[width=0.48\textwidth]{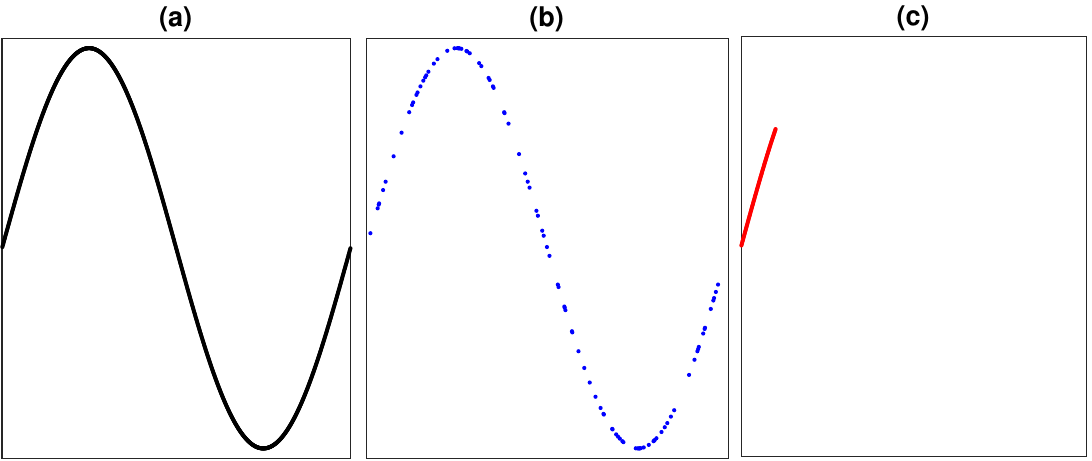}\vspace{-0.1in}
\caption{On the sampling condition for solving TCAS. (a) A 1000-dimensional signal we wish to recover. (b) 100 observations selected uniformly at random. (c) 100 observations chosen according to the setup of forecasting. While it should be doubtless that the signal in (a) can be reconstructed based on the observations in (b), it is not so real that the observations in (c) are sufficient for identifying the signal in (a).}\label{fig:complexity}
\end{center}
\end{figure}

We will prove that $\mathrm{DFT}_{\ell_1}$ strictly succeeds in recovering the target $\mathbf{L}_0$, as long as the sampling complexity, $\rho_0 = \mathrm{card}(\Omega)/m$ ($m=\Pi_{i=1}^nm_i$ and $\mathrm{card}(\Omega)$ is the number of observed entries), satisfies
\begin{align}
\rho_0 >1-\mathcal{O}(1/\|\mathcal{F}(\mathbf{L}_0)\|_0),
\end{align}
where $\|\cdot\|_0$ is the $\ell_0$ norm, i.e., the number of nonzero entries of a tensor. As one may have noticed, the complexity $1-\mathcal{O}(1/\|\mathcal{F}(\mathbf{L}_0)\|_0)$ could be very close to 1. Such a seemingly ``cumbersome'' sampling condition, in fact, is not caused by the limits of mathematical proofs but instead due to the nature of the setup of arbitrary sampling, as illustrated in Figure~\ref{fig:complexity}. Notably, the condition is useful for forecasting, as it gives a reasonable answer to Question~\ref{ques:sc}. Consider the setup of Problem~\ref{pb:mtsf}, where the observed samples are complete. Whenever $\mathcal{F}(\mathbf{L}_0)$ is sparse in a sense that $\|\mathcal{F}(\mathbf{L}_0)\|_0\leq\mathcal{O}(m^{\beta})$ with $0\leq\beta<1$, the condition $\rho_0 >1-\mathcal{O}(1/\|\mathcal{F}(\mathbf{L}_0)\|_0)$ requires
\begin{align}
p>\mathcal{O}(\tilde{m}^{\beta/(1-\beta)}h^{1/(1-\beta)}),
\end{align}
where $\tilde{m}=\Pi_{j=1}^{n-1}m_j$ is the dimension of the tensor $\mathbf{M}_t$, $1\leq{}t\leq{}p+h$. In other words, the minimum sampling size grows polynomially as the forecast horizon $h$, and thus long-term forecasting is entirely possible---but would be rather challenging.

While theoretically effective and computationally efficient, $\mathrm{DFT}_{\ell_1}$ suffers from a drawback that the time dimension is treated in the same way as the non-time dimensions, probably leading to imperfect results. To achieve better recovery performance, we further propose a regularizer called \emph{convolution nuclear norm}, which is the sum of \emph{convolution eigenvalues}~\citep{liu:tip:2014}, and which are generalization of Fourier frequencies. Specifically, the convolution nuclear norm of a tensor is the \emph{nuclear norm}~\cite{phd_2002_nuclear,siam_2010_minirank} of its \emph{convolution matrix} (see Section~\ref{sec:convmtx}), depicting the Tucker low-rankness of the convolution matrix of the tensor---this is the so-called \emph{convolutional low-rankness}. The derived method, Convolution Nuclear Norm Minimization (CNNM), attempts to tackle TCAS by
\begin{align}\label{eq:tc:cnnm:exact}
&\min_{\mathbf{L}} \norm{\mathcal{A}_k(\mathbf{L})}_{*}, \textrm{ s.t. }\mathcal{P}_{\Omega}(\mathbf{L}-\mathbf{L}_0)=0,
\end{align}
where $\|\cdot\|_*$ denotes the nuclear norm of a matrix, $\mathcal{A}_k(\cdot)$ is a linear map from $\mathbb{R}^{m_1\times\cdots\times{}m_n}$ to $\mathbb{R}^{m\times{}k}$ ($m=\Pi_{i=1}^nm_i$) such that $\mathcal{A}_k(\mathbf{L})$ produces the convolution matrix of $\mathbf{L}$, $k=\Pi_{i=1}^nk_i$ is the kernel size used in defining $\mathcal{A}_k(\cdot)$,
and $\{k_i: 1\leq{}k_i\leq{}m_i\}_{i=1}^n$ is a set of $n$ hyper-parameters (usually $k_i=\mathcal{O}(m_i)$). When $k_i=m_i,\forall{}i$, the convolution nuclear norm is exactly the $\ell_1$ norm of the Fourier transform, and thus CNNM falls back to $\mathrm{DFT}_{\ell_1}$.

As expected, CNNM also guarantees to recover the target $\mathbf{L}_0$ from a subset of its entries chosen arbitrarily, as long as
\begin{align}
\rho_0>1-\mathcal{O}(1/r_k(\mathbf{L}_0)),
\end{align}
where $r_k(\cdot)$, called \emph{convolution rank}, is the rank of the convolution matrix of a tensor. Generally speaking, the lower bound of $\rho_0$, called the \emph{sampling bound}, is functionally dependent on the kernel size $k$. Whenever $k_i=m_i,\forall{}i$ and thus $k=m$, we have $r_m(\mathbf{L}_0)=\|\mathcal{F}(\mathbf{L}_0)\|_0$ (see Section~\ref{sec:dftconv}). In this case, the sampling condition also rolls back to $\rho_0 >1-\mathcal{O}(1/\|\mathcal{F}(\mathbf{L}_0)\|_0)$. By choosing proper kernel sizes with an attempt to minimize the sampling bound (see Section~\ref{sc:main:diss}), CNNM may outperform dramatically $\mathrm{DFT}_{\ell_1}$, in terms of recovery accuracy. To summarize, the contributions of this paper mainly include:
\begin{itemize}
\item[$\diamond$] We explore the TSF problem from the view point of a new, inclusive problem named TCAS, and propose a novel method termed CNNM for studying TCAS, accomplishing an effective method for TSF as well. Remarkably, for the first time, our analysis provides a meaningful answer to the fundamental question of how many historical samples are required for predicting a given number of future samples.
\item[$\diamond$] Even though the general $\ell_1$ minimization problem in~\eqref{eq:dftl1:equi} was investigated in a great many papers, $\mathrm{DFT}_{\ell_1}$ has not been carefully explored before. To our knowledge, we are the first to prove that $\mathrm{DFT}_{\ell_1}$ strictly succeeds in solving TCAS, as long as the sampling complexity $\rho_0$ surpasses certain threshold.
\item[$\diamond$] The proposed convolutional low-rankness provides a new tool for generalizing and penetrating the classic concept of Fourier sparsity. Whenever a signal is smooth in the sense of circular shift, it is provable that the convolution matrix of the signal is approximately low-rank (see Section~\ref{sec:convlowrankness}). This result helps to explain the widely observed phenomenon that smooth, bounded signals (e.g., images and videos) are often approximately sparse in the Fourier domain.
\end{itemize}

The rest of this paper is organized as follows. Section~\ref{sec:note_pre} summarizes the mathematical notations used throughout the paper and introduces some preliminary knowledge, explaining the technical insights behind the proposed methods as well. Section~\ref{sec:cnnm} is mainly consist of theoretical analysis. Section~\ref{sec:proof} shows the mathematical proofs of the presented theorems. Section~\ref{sec:exp} demonstrates empirical results and Section~\ref{sec:con} concludes this paper.
\section{Notations and Preliminaries}\label{sec:note_pre}
\subsection{Summary of Main Notations}
Bold capital letters such as $\mathbf{X}$ denote order-$n$ tensors ($n\geq1$), and single numbers are denoted by either lowercase or Greek letters. Capital letters (e.g., $X$) and bold lowercase letters (e.g., $\mathbf{x}$) are used to represent matrices and vectors, respectively. For an order-$n$ tensor $\mathbf{X}$, $[\mathbf{X}]_{i_1,\cdots{},i_n}$ is the $(i_1,\cdots,i_n)$th entry of $\mathbf{X}$. Two types of tensor norms are used frequently throughout the paper: the Frobenius norm given by $\|\mathbf{X}\|_F=\sqrt{\sum_{i_1,\cdots{},i_n}|[\mathbf{X}]_{i_1,\cdots{},i_n}|^2}$, and the $\ell_1$ norm denoted by $\|\cdot\|_1$ and given by $\|\mathbf{X}\|_1=\sum_{i_1,\cdots{},i_n}|[\mathbf{X}]_{i_1,\cdots{},i_n}|$, where $|\cdot|$ denotes the magnitude of a real or complex number. Another two frequently used norms are the \emph{operator norm} and \emph{nuclear norm}~\cite{phd_2002_nuclear,siam_2010_minirank} of order-2 tensors (i.e., matrices), denoted by $\|\cdot\|$ and $\|\cdot\|_*$, respectively.

Calligraphic letters, such as $\mathcal{F}$, $\mathcal{P}$ and $\mathcal{A}$, are used to denote linear operators. In particular, $\mathcal{I}$ denotes the identity operator and $\Id$ is the identity matrix. For a linear operator $\mathcal{L}: \mathbb{H}_1\rightarrow\mathbb{H}_2$ between Hilbert spaces, its Hermitian adjoint (or conjugate) is denoted as $\mathcal{L}^*$ and given by
\begin{align}\label{eq:adjoint}
&\langle\mathcal{L}(\mathbf{X}), \mathbf{Y}\rangle_{\mathbb{H}_2} = \langle{}\mathbf{X}, \mathcal{L}^*(\mathbf{Y})\rangle_{\mathbb{H}_1}, \forall{}\mathbf{X}\in\mathbb{H}_1, \mathbf{Y}\in\mathbb{H}_2,
\end{align}
where $\langle\cdot,\cdot\rangle_{\mathbb{H}_i}$ is the inner product in the Hilbert space $\mathbb{H}_i$ ($i=1$ or $2$). But the subscript is omitted whenever $\mathbb{H}_i$ refers to a Euclidian space.

The symbol $\Omega$ is reserved to denote the sampling set consisting of the locations of observed entries. For $\Omega\subset\{1,\cdots,m_1\}$ $\times\cdots\times\{1,\cdots,m_n\}$, its \emph{mask tensor} is denoted by $\mathbf{\Theta}_{\Omega}$ and given by
\begin{align}
[\mathbf{\Theta}_\Omega]_{i_1,\cdots{},i_n}=\left\{\begin{array}{cc}
1,&\text{if }(i_1,\cdots,i_n)\in\Omega,\\
0, &\text{otherwise.}\\
\end{array}\right.
\end{align}
Denote by $\mathcal{P}_{\Omega}$ the orthogonal projection onto $\Omega$. Then we have the following:
\begin{align}\label{eq:omega}
&\mathcal{P}_\Omega(\mathbf{X}) = \mathbf{\Theta}_{\Omega}\circ{}\mathbf{X} \quad\textrm{and}\quad\Omega = \mathrm{supp}(\mathbf{\Theta}_{\Omega}),
\end{align}
where $\circ$ denotes the Hadamard product and $\mathrm{supp}(\cdot)$ is the support set of a tensor.

In most cases, we work with real-valued matrices (i.e., order-2 tensors). For a matrix $X$, $[X]_{i,:}$ is its $i$th row, and $[X]_{:,j}$ is its $j$th column. Let $\omega=\{j_1,\cdots,j_l\}$ be a 1D sampling set. Then $[X]_{\omega,:}$ denotes the submatrix of $X$ obtained by selecting the rows with indices $j_1,\cdots,j_l$, and similarly for $[X]_{:,\omega}$. For a 2D sampling set $\bar{\Omega}\subset{}\{1,\cdots,\bar{m}_1\} \times\{1,\cdots,\bar{m}_2\}$, we imagine it as a sparse matrix and define its ``rows'', ``columns'' and ``transpose'' accordingly. The $i$th row of $\bar{\Omega}$ is denoted by $\bar{\Omega}_i$ and given by $\bar{\Omega}_i = \{i_2: (i_1,i_2)\in\bar{\Omega}, i_1 = i\}$, the $j$th column is defined as $\bar{\Omega}^j = \{i_1: (i_1,i_2)\in\bar{\Omega}, i_2 = j\}$, and the transpose is given by $\bar{\Omega}^T = \{(i_2,i_1): (i_1,i_2)\in\bar{\Omega}\}$.

Letters $U$ and $V$ are reserved for the left and right singular vectors of matrices, respectively. The orthogonal projection onto the column space $U$ is denoted by $\mathcal{P}_U$ and given by $\mathcal{P}_U(X)=UU^TX$, and similarly for the row space $\mathcal{P}_V(X)=XVV^T$. The same notation is also used to represent a subspace of matrices, e.g., we say that $X\in\mathcal{P}_U$ for any matrix $X$ obeying $\mathcal{P}_U(X)=X$.
\subsection{Multi-Dimensional DFT}\label{sec:dfttran}
Multi-dimensional DFT, also known as multi-directional DFT, is one of the most widely-used tool for signal processing. Its definition is commonly available in some public documents. Here, we shall present a definition that would be easy for engineers to understand. First consider the case of $n=1$, i.e., the DFT of a vector $\mathbf{x}\in\mathbb{R}^m$. In this particular case, $\mathcal{F}(\mathbf{x})$ can be simply expressed as
\begin{align}
\mathcal{F}(\mathbf{x}) = U\mathbf{x},
\end{align}
with $U\in\mathbb{C}^{m\times{}m}$ being a complex-valued, symmetric matrix that satisfies $U^HU = UU^H=m\Id$, where $(\cdot)^H$ is the conjugate transpose of a complex-valued matrix. The matrix $U$, called the \emph{Fourier transform matrix}, is consist of data-independent numbers. More precisely, the real and imaginary components of $[U]_{i,j}$ are $\cos(2\pi(i-1)(j-1)/m)$ and $-\sin(2\pi(i-1)(j-1)/m)$, respectively.

Similarly, when $n=2$, the DFT of a matrix $X\in\mathbb{R}^{m_1\times{}m_2}$ is given by
\begin{align}
\mathcal{F}(X) = U_1XU_2,
\end{align}
where $U_1\in\mathbb{C}^{m_1\times{}m_1}$ and $U_2\in\mathbb{C}^{m_2\times{}m_2}$ are the Fourier transform matrices for the columns and rows of $X$, respectively. For a general order-$n$ ($n\geq1$) tensor $\mathbf{X}\in\mathbb{R}^{m_1\times\cdots\times{}m_n}$, its DFT is given by
\begin{align}
\mathcal{F}(\mathbf{X}) = \mathbf{X}\times_1{}U_1\cdots\times_n{}U_n,
\end{align}
where $U_i\in\mathbb{C}^{m_i\times{}m_i}$ is the Fourier transform matrix for the $i$th dimension, and $\times_j$ ($1\leq{}j\leq{}n$) is the \emph{mode-$j$ product}~\cite{hosvd:siam:2000} between tensors and matrices.
\subsection{Convolution Matrix}\label{sec:convmtx}
Discrete convolution, which is probably the most fundamental concept in signal processing, plays an important role in this paper. Its definition---though mostly unique---has multiple variants, depending on which \emph{boundary condition} is used. What we consider in this paper is the \emph{circular convolution}, i.e., the convolution equipped with \emph{circulant boundary condition}~\cite{mamdouh:sip:2012}. Let's begin with the simple case of $n=1$, i.e., the circular convolution procedure of converting $\mathbf{x}\in\mathbb{R}^m$ and $\mathbf{k}\in\mathbb{R}^k$ ($k\leq{}m$) into $\mathbf{x}\star{}\mathbf{k}\in\mathbb{R}^m$:
\begin{align}
[\mathbf{x}\star{}\mathbf{k}]_i = \sum_{s=1}^k [\mathbf{x}]_{i-s}[\mathbf{k}]_{s}, i = 1,\cdots,m,
\end{align}
where $\star$ denotes the convolution operator, and it is assumed that $[\mathbf{x}]_{i-s} = [\mathbf{x}]_{i-s+m}$ for $i\leq{}s$; this is the so-called circulant boundary condition. Throughout this paper, we assume $k\leq{}m$ and refer to the smaller tensor $\mathbf{k}$ as the \emph{kernel}. In general, the convolution operator is linear and can be converted into matrix multiplication:
\begin{align}
\mathbf{x}\star{}\mathbf{k} = \mathcal{A}_k(\mathbf{x})\mathbf{k}, \forall{}\mathbf{x}\in\Re^m, \mathbf{k}\in\Re^k,
\end{align}
where $\mathcal{A}_k(\cdot)$ is the \emph{convolution matrix} of a vector, and the subscript $k$ is put to remind the readers that the convolution matrix is always associated with a certain kernel size $k$. In the light of circular convolution, the convolution matrix of a vector $\mathbf{x}=[x_1,\cdots,x_m]^T$ (i.e., $[\mathbf{x}]_{i}=x_i$) is a truncated circular matrix of size $m\times{}k$:
\begin{align}
\mathcal{A}_k(\mathbf{x}) = \left[\begin{array}{cccc}
x_1 & x_m&\cdots&x_{m-k+2}\\
x_2 & x_1&\cdots&x_{m-k+3}\\
\vdots&\vdots&\vdots&\vdots\\
x_m & x_{m-1}&\cdots&x_{m-k+1}
\end{array}\right].
\end{align}
In other words, the $j$th column of $\mathcal{A}_k(\mathbf{x})$ is given by $\mathcal{S}^{j-1}(\mathbf{x})$, where $\mathcal{S}$ is an operator that circularly shifts the elements of a vector by one position; namely,
\begin{align}\label{eq:circshift}
\mathcal{S}(\mathbf{x}) = [x_m,x_1,x_2,\cdots,x_{m-1}]^T.
\end{align}
This shift operator is implemented by the Matlab function ``circshift''. In the special case of $k=m$, the convolution matrix $\mathcal{A}_m(\mathbf{x})$ is exactly an $m\times{}m$ circular matrix.

Now we turn to the general case of any $n\geq1$. Let $\mathbf{X}\in\mathbb{R}^{m_1\times\cdots\times{}m_n}$ and $\mathbf{K}\in\mathbb{R}^{k_1\times\cdots\times{}k_n}$ be two order-$n$ tensors. Again, it is configured that $k_j\leq{}m_j, \forall{}1\leq{}j\leq{}n$, and $\mathbf{K}$ is called the kernel. Then the procedure of circularly convoluting $\mathbf{X}$ and $\mathbf{K}$ into $\mathbf{X}\star{}\mathbf{K}\in\mathbb{R}^{m_1\times\cdots\times{}m_n}$ is given by
\begin{align}
[\mathbf{X}\star{}\mathbf{K}]_{i_1,\cdots{},i_n} = \sum_{s_1,\cdots,s_n}[\mathbf{X}]_{i_1-s_1,\cdots{},i_n-s_n}[\mathbf{K}]_{s_1,\cdots,s_n}.
\end{align}
The above convolution procedure can be also converted into matrix multiplication. Let $\mathrm{vec}(\cdot)$ be the vectorization of a tensor, then we have
\begin{align}\label{eq:conv2mtxproduct}
&\mathrm{vec}(\mathbf{X}\star{}\mathbf{K}) = \mathcal{A}_k(\mathbf{X})\mathrm{vec}(\mathbf{K}), \forall{}\mathbf{X}, \mathbf{K},
\end{align}
where $\mathcal{A}_k(\mathbf{X})\in\Re^{m\times{}k}$, with $m=\Pi_{j=1}^nm_j$ and $k=\Pi_{j=1}^nk_j$, is the convolution matrix of the tensor $\mathbf{X}$. To compute the convolution matrix of an order-$n$ tensor, one just needs to replace the one-directional circular shift operator given in~\eqref{eq:circshift} with a multi-directional one, so as to stay in step with the structure of high-order tensors. To be more precise, let $\mathcal{S}(\mathbf{X},u,v)$ be the operator that circularly shifts the elements in the tensor $\mathbf{X}$ by $u$ positions along the $v$th dimension (or direction); this operator is also implemented by the Matlab function ``circshift''. For any index $(t_1,\cdots,t_n)$ with $1\leq{}t_q\leq{}k_q$ ($1\leq{}q\leq{}n$), define an invertible operator $\mathcal{T}_{(t_1,\cdots,t_n)}: \mathbb{R}^{m_1\times\cdots\times{}m_n}\rightarrow\mathbb{R}^m$ as
\begin{align}
 \mathcal{T}_{(t_1,\cdots,t_n)}(\mathbf{X}) = \mathrm{vec}(\mathbf{X}_n),
\end{align}
where $\mathbf{X}_n$ is determined by the following recursion rule:
\begin{align}\label{eq:multi-shift}
\mathbf{X}_q=\mathcal{S}(\mathbf{X}_{q-1},t_q-1,q), 1\leq{}q\leq{}n, \mathbf{X}_0= \mathbf{X}.
\end{align}
Assume conveniently that $k_0 = 1$, and let $j=1+\sum_{q=1}^{n}(t_q-1)\Pi_{b=0}^{q-1}k_b$. Then it may be easily seen that $j$ ranges from 1 to $k$ while $t_q$ increases from 1 to $k_q$, $q=1,\cdots,n$, and the $j$th column of $\mathcal{A}_k(\mathbf{X})$ is given by
\begin{align}\label{eq:convmtx:tensor}
[\mathcal{A}_k(\mathbf{X})]_{:,j} = \mathcal{T}_{(t_1,\cdots,t_n)}(\mathbf{X}).
\end{align}
As a result, when $k_i=m_i$, $\forall{}i$, the convolution matrix of an order-$n$ tensor is a \emph{block-circular matrix}~\cite{MonC2008} taking a form as in the following (up to some permutation):
\begin{align}\label{eq:blockcirc}
\left[\begin{array}{cccc}
C_1 & C_{\bar{m}}&\cdots&C_2\\
C_2 & C_1&\cdots&C_3\\
\vdots&\vdots&\vdots&\vdots\\
C_{\bar{m}-1}&C_{\bar{m}-2}&\cdots&C_{\bar{m}}\\
C_{\bar{m}} & C_{\bar{m}-1}&\cdots&C_1
\end{array}\right],
\end{align}
where $C_i\in\mathbb{R}^{m_1\times{}m_1}$ ($1\leq{}i\leq{}\bar{m}$) is also a circular matrix and $\bar{m}=m/m_1$. The general case of $k_i\leq{}m_i$ is similar. Namely, $\mathcal{A}_k(\cdot)$ is a truncated block-circular matrix consisting of $m/m_1$ row partitions and $k/k_1$ column partitions, with each block being a truncated circular matrix of size $m_1\times{}k_1$.
\subsection{Connections Between DFT and Convolution}\label{sec:dftconv}
Whenever the kernel $\mathbf{K}$ has the same size as the tensor $\mathbf{X}$, i.e., $k_j=m_j$, $\forall{}1\leq{}j\leq{}n$, the resultant convolution matrix, $\mathcal{A}_m(\mathbf{X})$, is block-circular and therefore can be diagonalized by DFT. The cases of $n=1$ and $n=2$ are well-known and have been widely used in the literature, e.g.,~\cite{Ng:1999:FAD,wang:siam:2008}. In effect, the conclusion holds for any $n\geq1$, as pointed out by Proposition 2.5 of~\cite{MonC2008}. More precisely, let the DFT of $\mathbf{X}$ be $\mathcal{F}(\mathbf{X}) = \mathbf{X}\times_1{}U_1\cdots\times_n{}U_n$, and denote $U = U_1\otimes\cdots\otimes{}U_n$ with $\otimes$ being the Kronecker product. Then $U\mathcal{A}_m(\mathbf{X})U^H$ is a diagonal matrix, namely $U\mathcal{A}_m(\mathbf{X})U^H=m\Sigma$ with $\Sigma=\diag{\sigma_1,\cdots,\sigma_m}$. Since the first column of $U$ is a vector of all ones, it is easy to see that
\begin{align}
&\mathrm{vec}(\mathcal{F}(\mathbf{X})) = U\mathrm{vec}(\mathbf{X}) = U[\mathcal{A}_m(\mathbf{X})]_{:,1} \\\nonumber
&=[U\mathcal{A}_m(\mathbf{X})]_{:,1} = [\Sigma{}U]_{:,1} = [\sigma_1,\cdots,\sigma_m]^T.
\end{align}
That is, the eigenvalues of the convolution matrix $\mathcal{A}_m(\mathbf{X})$ are exactly the Fourier frequencies given by $\mathcal{F}(\mathbf{X})$. Hence, for any $\mathbf{X}\in\mathbb{R}^{m_1\times\cdots\times{}m_n}$, we have
\begin{align}\label{eq:dftnucconn}
\norm{\mathcal{F}(\mathbf{X})}_0 = \rank{\mathcal{A}_m(\mathbf{X})},\textrm{ }\norm{\mathcal{F}(\mathbf{X})}_1 = \norm{\mathcal{A}_m(\mathbf{X})}_{*},
\end{align}
where $\|\cdot\|_0$ is the $\ell_0$ norm of a tensor, and $\|\cdot\|_*$ is the nuclear norm of a matrix. As a consequence, the $\mathrm{DFT}_{\ell_1}$ program~\eqref{eq:tc:dftl1:exact} is equivalent to the following real-valued convex optimization problem:
\begin{align}\label{eq:tc:dftl1:eqiv}
&\min_{\mathbf{L}} \norm{\mathcal{A}_m(\mathbf{L})}_{*}, \textrm{ s.t. }\mathcal{P}_{\Omega}(\mathbf{L}-\mathbf{L}_0)=0.
\end{align}
Hence, $\mathrm{DFT}_{\ell_1}$ is a special case of CNNM. Although real-valued and convex, the above problem is hard to be solved efficiently, thereby the formulation is used only for the purpose of theoretical analysis.
\begin{figure}[h!]
\begin{center}
\includegraphics[width=0.48\textwidth]{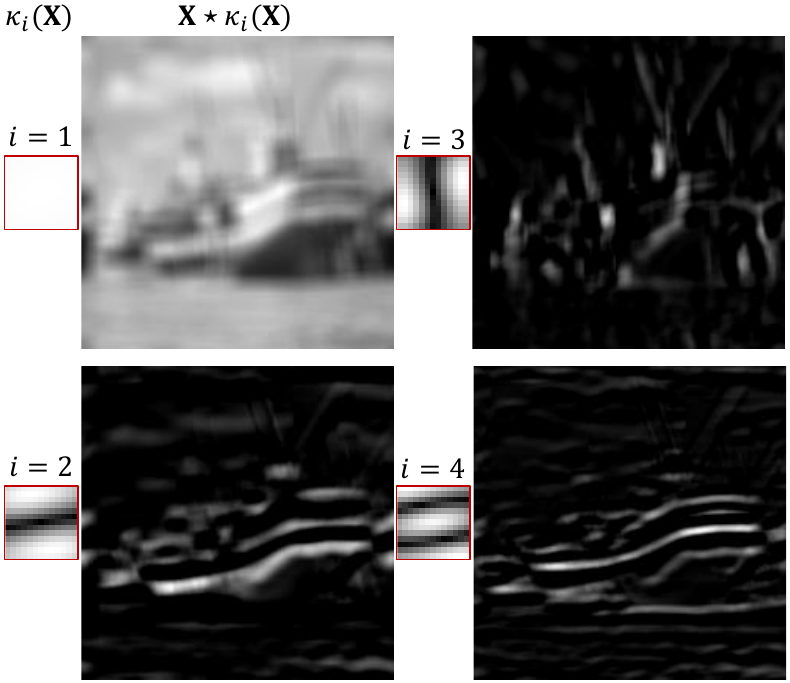}\vspace{-0.1in}
\caption{Visualizing the convolution eigenvectors and filtered signals, using a $200\times200$ Boats image as the experimental data and setting $k_1=k_2=13$. For the sake of visual effects, the convolution eigenvector $\kappa_i(\mathbf{X})$ is processed such that its values are in the range of 0 and 1, and the filtered signal $\mathbf{X}\star\kappa_i(\mathbf{X})$ is normalized to have a maximum of 1.}\label{fig:ceigs}
\end{center}
\end{figure}
\subsection{Convolution Eigenvalues}
The concept of convolution eigenvalues is first proposed and investigated by~\cite{liu:tip:2014}, under the context of image deblurring. Though made specific to order-2 tensors (i.e., matrices), the definitions given in~\cite{liu:tip:2014} can be easily generalized to order-$n$ tensors with any $n\geq1$.
\begin{defn}[Convolution Eigenvalues and Eigenvectors~\cite{liu:tip:2014}]\label{def:eigen} For a tensor $\mathbf{X}\in\mathbb{R}^{m_1\times\cdots\times{}m_n}$ associated with a certain kernel size $k_1\times\cdots\times{}k_n$ $(k_j\leq{}m_j$, $\forall{}j)$, its first convolution eigenvalue is denoted as $\sigma_1(\mathbf{X})$ and given by
\begin{eqnarray}
\sigma_1(\mathbf{X})=\max_{\mathbf{K}\in\mathbb{R}^{k_1\times\cdots\times{}k_n}}\|\mathbf{X}\star{}\mathbf{K}\|_F,\textrm{ s.t. } \|\mathbf{K}\|_F=1.
\end{eqnarray}
The maximizer to above problem is called the first convolution eigenvector, denoted as $\kappa_1(\mathbf{X})\in\mathbb{R}^{k_1\times\cdots\times{}k_n}$. Similarly, the $i$th ($i=2,\ldots,k$, $k=\Pi_{j=1}^nk_j$) convolution eigenvalue, denoted as $\sigma_i(\mathbf{X})$, is defined as
\begin{eqnarray}\sigma_i(\mathbf{X})&=&\max_{\mathbf{K}\in\mathbb{R}^{k_1\times\cdots\times{}k_n}}\|\mathbf{X}\star{}\mathbf{K}\|_F, \\\nonumber
 &\textrm{s.t.}& \|\mathbf{K}\|_F=1, \langle{}\mathbf{K},\kappa_j(\mathbf{X})\rangle=0,\forall{}j<i.
\end{eqnarray}
The maximizer to above problem is the $i$th convolution eigenvector, denoted as $\kappa_i(\mathbf{X})\in\mathbb{R}^{k_1\times\cdots\times{}k_n}$.
\end{defn}

Figure~\ref{fig:ceigs} shows the first four convolution eigenvectors of a natural image. It can be seen that the convolution eigenvectors are essentially a series of ordered filters, in which the later ones have higher cut-off frequencies. Note that, unlike the standard singular values, the convolution eigenvalues of a fixed tensor are not fixed and may vary along with the kernel size.

Due to the relationship given in~\eqref{eq:conv2mtxproduct}, the convolution eigenvalues are nothing more than the singular values of the convolution matrix. Thus, the number of nonzero convolution eigenvalues of a tensor, called the \emph{convolution rank}, is simply the rank of the convolution matrix of the tensor. The so-called \emph{convolution nuclear norm}, defined as the sum of convolution eigenvalues, strictly equals to the nuclear norm of the convolution matrix.
\subsection{Convolution Coherence}
The same as in most studies about matrix completion, we also need to access the concept of \emph{coherence}~\cite{Candes:2009:math,liu:tsp:2016}. For a rank-$r$ matrix $Y\in\mathbb{R}^{a\times{}b}$, let its skinny SVD be $Y=U\Sigma{}V^T$ with $\Sigma\in\mathbb{R}^{r\times{}r}$. Then there are two coherence parameters, $\mu_u$ and $\mu_v$, for characterizing some properties of $Y$:
\begin{align}
\mu_u(Y) = \frac{a}{r}\max_{1\leq{}i\leq{}a}\|[U]_{i,:}\|_F^2,\textrm{ }\mu_v(Y) = \frac{b}{r}\max_{1\leq{}i\leq{}b}\|[V]_{i,:}\|_F^2,
\end{align}
where $[\cdot]_{i,:}$ is the $i$th row of a matrix. Accordingly, the \emph{convolution coherence} of a tensor $\mathbf{X}\in\mathbb{R}^{m_1\times\cdots\times{}m_n}$, denoted as $\mu_k(\mathbf{X})$, is defined as the coherence of the convolution matrix $\mathcal{A}_k(\mathbf{X})$:
\begin{align}\label{eq:coherence}
\mu_k(\mathbf{X}) =\max(\mu_u(\mathcal{A}_k(\mathbf{X})), \mu_v(\mathcal{A}_k(\mathbf{X}))).
\end{align}
Let $m=\Pi_{i=1}^nm_i$ and $k=\Pi_{i=1}^nk_i$. By definitions, $1\leq\mu_u(\mathcal{A}_k(\mathbf{X}))\leq{}m$ and $1\leq{}\mu_v(\mathcal{A}_k(\mathbf{X}))\leq{}k$, which simply lead to $1\leq\mu_k(\mathbf{X})\leq{}m$. Yet, due to the special property of the convolution matrix, the upper bound $m$ is hardly attainable and a more accurate range is given by
\begin{align}
1\leq\mu_k(\mathbf{X})\leq\frac{m}{k}\nu^2,
\end{align}
where $\nu$ is the \emph{condition number} of $\mathcal{A}_k(\mathbf{X})$, i.e., $\nu$ is the ratio of the largest singular value of $\mathcal{A}_k(\mathbf{X})$ to its smallest nonzero singular value.
It can be verified that the above lower bound is achieved if $\mathbf{X}$ is a constant tensor whose elements are all the same, and the upper bound is attained if $\mathbf{X}$ is a standard basis of $\mathbb{R}^{m_1\times\cdots\times{}m_n}$, $\forall{}1\leq{}k\leq{}m$.
\begin{proof}Let the skinny SVD of $\mathcal{A}_k(\mathbf{X})$ be $\mathcal{A}_k(\mathbf{X})=U\Sigma{}V^T$, where $\Sigma=\diag{\sigma_1,\cdots,\sigma_r}$, and $\sigma_1\geq\cdots\geq\sigma_r>0$ are the nonzero singular values of $\mathcal{A}_k(\mathbf{X})$. Considering the $\ell_2$ norm of the $j$th column of $\mathcal{A}_k(\mathbf{X})$, we have
\begin{align}
\|\mathbf{X}\|_F^2\hspace{-0.05in}=\hspace{-0.05in}\|[U\Sigma{}V^T]_{:,j}\|_F^2\hspace{-0.05in}=\hspace{-0.05in}\|\Sigma{}[V^T]_{:,j}\|_F^2\geq{}\sigma_r^2\|[V]_{j,:}\|_F^2,
\end{align}
which gives that
\begin{align}
\|[V]_{j,:}\|_F^2\leq\frac{\|\mathbf{X}\|_F^2}{\sigma_r^2}=\frac{\sum_{i=1}^r\sigma_i^2}{k\sigma_r^2}\leq\frac{r}{k}\nu^2,
\end{align}
from which it follows that $\mu_v(\mathcal{A}_k(\mathbf{X}))\leq{}\nu^2$. Furthermore, since the $\ell_2$ norm of any row of $\mathcal{A}_k(\mathbf{X})$ is at most $\|\mathbf{X}\|_F$, we have $\mu_u(\mathcal{A}_k(\mathbf{X}))\leq{}m\nu^2/k$.
\end{proof}
\subsection{Convolution Sampling Set}\label{sec:cnnm:csample}
According to the theories in~\cite{liu:nips:2017,liu:tpami:2021}, it seems ``impossible'' to restore the target $\mathbf{L}_0$ when some of its slices are wholly missing. Nevertheless, as aforementioned, this assertion is made based on the premise that the target $\mathbf{L}_0$ has the lowest Tucker rank among all possible completions. In the light of convolutional low-rankness, recovery with arbitrary sampling pattern is indeed feasible.

First of all, we would like to introduce a concept called \emph{convolution sampling set}:
\begin{defn}[Convolution Sampling Set]\label{def:comg} For $\Omega\subset\{1,\cdots,$ $m_1\}\times\cdots\times\{1,\cdots,m_n\}$ with $m=\Pi_{j=1}^nm_j$, its convolution sampling set associated with kernel size $k_1\times\cdots\times{}k_n$ is denoted by $\Omega_{\mathcal{A}}$ and given by
\begin{align}\label{eq:conv:omega}
\Theta_{\Omega_{\mathcal{A}}} = \mathcal{A}_k(\mathbf{\Theta}_{\Omega}) \quad\textrm{and}\quad \Omega_{\mathcal{A}} = \mathrm{supp} (\Theta_{\Omega_{\mathcal{A}}}),
\end{align}
where $\mathbf{\Theta}_{\Omega}\in\Re^{m_1\times\cdots\times{}m_n}$ and $\Theta_{\Omega_{\mathcal{A}}}\in\mathbb{R}^{m\times{}k}$ are the mask tensor and the mask matrix of $\Omega$ and $\Omega_{\mathcal{A}}$, respectively. Note here that the subscript $k$ is omitted from $\Omega_{\mathcal{A}}$ for the sake of simplicity.
\end{defn}
\begin{figure}[h!]
\begin{center}
\includegraphics[width=0.4\textwidth]{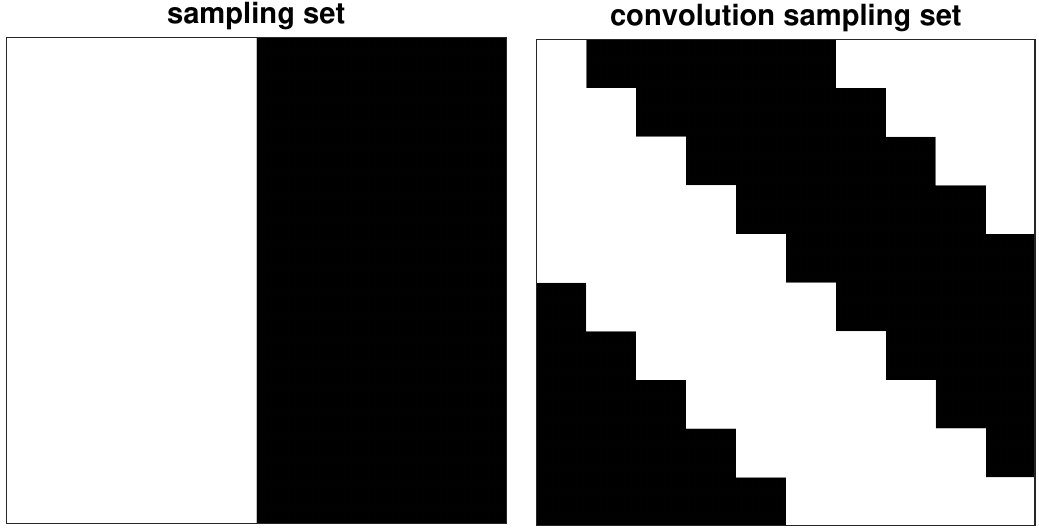}\vspace{-0.1in}
\caption{Illustrating the effects of convolution. Left: the mask $\mathbf{\Theta}_{\Omega}$ of a 2D sampling set $\Omega$ with $m_1=m_2=10$, where the last 5 columns are wholly missing. Right: the mask $\Theta_{\Omega_{\mathcal{A}}}$ of the convolution sampling set $\Omega_{\mathcal{A}}$ with $k_1=k_2=10$.}\label{fig:csample}
\end{center}
\end{figure}

In general, $\Omega_{\mathcal{A}}$ is a convolution counterpart of $\Omega$, and the corresponding orthogonal projection onto $\Omega_{\mathcal{A}}$ is given by $\mathcal{P}_{\Omega_{\mathcal{A}}}(Y) = \Theta_{\Omega_{\mathcal{A}}}\circ{}Y, \forall{}Y\in\mathbb{R}^{m\times{}k}$. \emph{No matter how the observed entries are selected, the convolution sampling set $\Omega_{\mathcal{A}}$ always exhibits a well-posed pattern}. Namely, each column of the mask matrix $\Theta_{\Omega_{\mathcal{A}}}$ has exactly $\rho_0m$ ones and $(1-\rho_0)m$ zeros, and each row of $\Theta_{\Omega_{\mathcal{A}}}$ has at most $(1-\rho_0)m$ zeros. Whenever $k_j=m_j,\forall{}j$, each row of $\Theta_{\Omega_{\mathcal{A}}}$ has also exactly $\rho_0m$ ones and $(1-\rho_0)m$ zeros, as shown in Figure~\ref{fig:csample}.

The following lemma shows some algebraic properties about $\Omega_{\mathcal{A}}$, playing a key role in the proofs.
\begin{lemm}\label{lem:alge:conv}
Let $\Omega\subset\{1,\cdots,m_1\}\times\cdots\times\{1,\cdots,m_n\}$, and let the kernel size used to define $\mathcal{A}_k$ be $k_1\times\cdots\times{}k_n$, $1\leq{}k_j\leq{}m_j, \forall{}1\leq{}j\leq{}n$. Denote $m=\Pi_{j=1}^nm_j$ and $k=\Pi_{j=1}^nk_j$. Let $\Omega_{\mathcal{A}}\subset\{1,\cdots,m\}\times\{1,\cdots,k\}$ be the 2D convolution sampling set defined in~\eqref{eq:conv:omega}. For any $Y\in\mathbb{R}^{m\times{}k}$ and $\mathbf{X}\in\mathbb{R}^{m_1\times\cdots\times{}m_n}$, we have the following:
\begin{align}
&\mathcal{A}_k^*\mathcal{A}_k(\mathbf{X}) = k\mathbf{X},\\\nonumber
&\mathcal{A}_k\mathcal{P}_{\Omega}(\mathbf{X}) = \mathcal{P}_{\Omega_{\mathcal{A}}}\mathcal{A}_k(\mathbf{X}), \\\nonumber
&\mathcal{A}_k^*\mathcal{P}_{\Omega_{\mathcal{A}}}(Y) = \mathcal{P}_{\Omega}\mathcal{A}_k^*(Y),
\end{align}
where $\mathcal{A}_k^*$ is the Hermitian adjoint of $\mathcal{A}_k$.
\end{lemm}

\subsection{Isomerism and Relative Well-Conditionedness}
By the definition of convolution matrix, recovering $\mathbf{L}_0$ can ensure to recover its convolution matrix $\mathcal{A}_k(\mathbf{L}_0)$. On the other hand, Lemma~\ref{lem:alge:conv} implies that obtaining $\mathcal{A}_k(\mathbf{L}_0)$ also suffices to identify $\mathbf{L}_0$. So, Problem~\ref{pb:tc} can be equivalently converted into a standard matrix completion problem:
\begin{prob}[Dual Problem]\label{pb:dual}
Use the same notations as in Problem~\ref{pb:tc}. Denote by $\Omega_{\mathcal{A}}$ the convolution sampling set of $\Omega$. Given $\mathcal{P}_{\Omega_{\mathcal{A}}}(\mathcal{A}_k(\mathbf{L}_0))$, the goal is to recover $\mathcal{A}_k(\mathbf{L}_0)$.
\end{prob}

As aforementioned, the pattern of $\Omega_{\mathcal{A}}$ is always well-posed, in a sense that some observations are available at every column and row of the matrix.\footnote{To meet this, the kernel size needs be chosen properly. For example, under the setup of Problem~\ref{pb:mtsf}, $k_n$ should be greater than $h$.} Hence, provided that $\mathbf{L}_0$ is convolutionally low-rank, i.e., $\mathcal{A}_k(\mathbf{L}_0)$ is low-rank, Problem~\ref{pb:dual} is exactly the LRMC problem widely studied in the literature~\cite{Candes:2009:math,gross:tit:2011,ge:nips:2016,liu:tsp:2016,yoon2018icml}. However, unlike the setting of random sampling adopted by most studies, the sampling regime here is \emph{deterministic} rather than random, thereby we have to count on the techniques established by~\cite{liu:nips:2017,liu:tpami:2021}. For the completeness of presentation, we would briefly introduce the concepts of \emph{isomeric condition} (or \emph{isomerism})~\cite{liu:nips:2017} and \emph{relative well-conditionedness}~\cite{liu:tpami:2021}.
\begin{defn}[$\bar{\Omega}/\bar{\Omega}^T$-Isomeric~\cite{liu:nips:2017}]\label{def:iso:omg}
Let $X\in\Re^{a\times{}b}$ be a matrix and $\bar{\Omega}\subset\{1,\cdots,a\}\times\{1,\cdots,b\}$ be a 2D sampling set. Suppose that $\bar{\Omega}_i\neq\emptyset$ (empty set) and $\bar{\Omega}^j\neq\emptyset$, $\forall{}i,j$. Then $X$ is $\bar{\Omega}$-isomeric iff
\begin{align}
&\rank{[X]_{\bar{\Omega}^j,:}} = \rank{X}, \forall{}j = 1,\cdots,b.
\end{align}
Furthermore, the matrix $X$ is called $\bar{\Omega}/\bar{\Omega}^T$-isomeric iff $X$ is $\bar{\Omega}$-isomeric and $X^T$ is $\bar{\Omega}^T$-isomeric.
\end{defn}
\begin{defn}[$\bar{\Omega}/\bar{\Omega}^T$-Relative Condition Number~\cite{liu:tpami:2021}]\label{def:rcn}
Use the same notations as in Definition~\ref{def:iso:omg}. Suppose that $[X]_{\bar{\Omega}^j,:}\neq0$ and $[X]_{:,\bar{\Omega}_i}\neq0$, $\forall{}i,j$. Then the $\bar{\Omega}$-relative condition number of $X$ is denoted by $\gamma_{\bar{\Omega}}(X)$ and given by
\begin{align}
\gamma_{\bar{\Omega}}(X) = \min_{1\leq{}j\leq{}b}1/\|X([X]_{\bar{\Omega}^j,:})^+\|^2,
\end{align}
where $(\cdot)^+$ is the Moore-Penrose pseudo-inverse of a matrix. Furthermore, the $\bar{\Omega}/\bar{\Omega}^T$-relative condition number of $X$ is denoted by $\gamma_{\bar{\Omega},\bar{\Omega}^T}(X)$ and given by $\gamma_{\bar{\Omega},\bar{\Omega}^T}(X) = \min(\gamma_{\bar{\Omega}}(X), \gamma_{\bar{\Omega}^T}(X^T))$.
\end{defn}

In order to show that CNNM succeeds in recovering $\mathbf{L}_0$ even when the observed entries are arbitrarily placed, we just need to prove that $\mathcal{A}_k(\mathbf{L}_0)$ is $\Omega_\mathcal{A}/\Omega_\mathcal{A}^T$-isomeric and $\gamma_{\Omega_\mathcal{A},\Omega_\mathcal{A}^T}(\mathcal{A}_k(\mathbf{L}_0))$ is reasonably large as well. To do this, the following lemma is useful.
\begin{lemm}[\cite{liu:tpami:2021}]\label{lem:iso:rcn}
Use the same notations as in Definition~\ref{def:iso:omg}. Let $\mu_X=\max(\mu_u(X), \mu_v(X))$ be the coherence of the matrix $X$, and let $r_X$ be the rank of $X$. Define a quantity $\rho$ as
\begin{align}
\rho = \min(\min_{1\leq{}i\leq{}a}\mathrm{card}(\bar{\Omega}_i)/b, \min_{1\leq{}j\leq{}b}\mathrm{card}(\bar{\Omega}^j)/a).
\end{align}
For any $0\leq\alpha<1$, if $\rho>1-(1-\alpha)/(\mu_Xr_X)$ then $X$ is $\bar{\Omega}/\bar{\Omega}^T$-isomeric and $\gamma_{\bar{\Omega},\bar{\Omega}^T}(X)>\alpha$.
\end{lemm}

\section{Analysis and Algorithms}\label{sec:cnnm}
In this section, we will provide theoretical analysis to validate the recovery ability of CNNM (and $\mathrm{DFT}_{\ell_1}$), uncovering the mystery on why CNNM can work with arbitrary sampling patterns.

\subsection{Main Results}\label{sec:main:result}
First consider the ideal case where the observed data is precise and noiseless. In this case, the CNNM program~\eqref{eq:tc:cnnm:exact} guarantees to exactly recover the target $\mathbf{L}_0$ under a certain sampling condition, as shown in the following theorem.
\begin{theo}[Noiseless]\label{main:thm:cnnm:noiseless}
Let $\mathbf{L}_0\in\mathbb{R}^{m_1\times{}\cdots\times{}m_n}$ and $\Omega\subset\{1,\cdots,m_1\}\times\cdots\times\{1,\cdots,m_n\}$. Let the adopted kernel size be $k_1\times{}\cdots\times{}k_n$ with $k_j\leq{}m_j,\forall{}1\leq{}j\leq{}n$. Denote $k=\Pi_{j=1}^nk_j$, $m=\Pi_{j=1}^nm_j$ and $\rho_0=\mathrm{card}(\Omega)/m$. Denote by $r_k(\mathbf{L}_0)$ and $\mu_k(\mathbf{L}_0)$ the convolution rank and convolution coherence of the target $\mathbf{L}_0$, respectively. Then $\mathbf{L}=\mathbf{L}_0$ is the unique minimizer to the CNNM program~\eqref{eq:tc:cnnm:exact}, as long as
\begin{align}
\rho_0 > 1 - \frac{0.25k}{\mu_k(\mathbf{L}_0)r_k(\mathbf{L}_0)m}.
\end{align}
\end{theo}

The above theorem illustrates that, to maximize the recovery ability of CNNM, the kernel size $k$ should be chosen to minimize the sampling bound $1 - 0.25k/(\mu_k(\mathbf{L}_0)$ $r_k(\mathbf{L}_0)m)$, which suggests to minimize $r_k(\mathbf{L}_0)/k$. Note here that $r_k(\mathbf{L}_0)\leq{}k$ and $r_k(\mathbf{L}_0)$ is a non-decreasing function of $k$, thereby $r_k(\mathbf{L}_0)/k$ could be a U-shaped function of $k$. This is a useful clue for determining the key parameters $\{k_1,\cdots,k_n\}$ in realistic environments, as we will elaborate in Section~\ref{sc:main:diss}. By setting the kernel to have the same size with the target $\mathbf{L}_0$, CNNM falls back to $\mathrm{DFT}_{\ell_1}$. Thus, the following is an immediate consequence of Theorem~\ref{main:thm:cnnm:noiseless}.
\begin{coro}[Noiseless]\label{main:coro:dft:noiseless}
Use the same notations as in Theorem~\ref{main:thm:cnnm:noiseless}, and set $k_j=m_j,\forall{}1\leq{}j\leq{}n$. Then $\mathbf{L}=\mathbf{L}_0$ is the unique minimizer to the $\mathrm{DFT}_{\ell_1}$ program~\eqref{eq:tc:dftl1:exact}, as long as
\begin{align}
\rho_0 > 1 - \frac{0.25}{\mu_m(\mathbf{L}_0)\|\mathcal{F}(\mathbf{L}_0)\|_0}.
\end{align}
\end{coro}
\begin{figure}[h!]
\begin{center}
\includegraphics[width=0.48\textwidth]{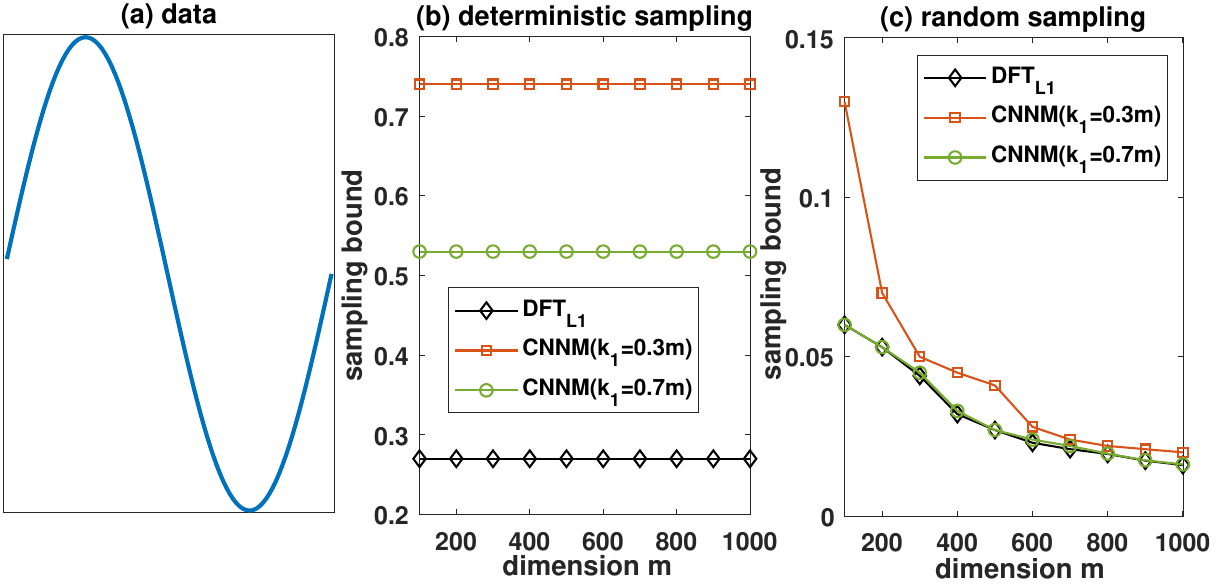}\vspace{-0.15in}
\caption{Investigating the difference between future data and randomly chosen missing entries. (a) The sine sequence used for experiments: $\{\mathbf{M}_t\}_{t=1}^{m}$ with $\mathbf{M}_t = \sin(2t\pi/m)$, thereby $\mathbf{L}_0$ is an $m$-dimensional vector with $r_k(\mathbf{L}_0)=2$ and $\mu_k(\mathbf{L}_0)=1, \forall{}m\geq{}k>2$. (b) The sampling bound under the setup of forecasting, where $\{\mathbf{M}_t\}_{t=\rho_0m+1}^{m}$ is the missing data. (c) The sampling bound under the context of random sampling. In these experiments, the sampling bound is calculated as the smallest fraction of observed entries for the methods to succeed in recovering $\mathbf{L}_0$, in a sense that the recovery accuracy measured by Peak Signal-to-Noise Ratio (PSNR) is greater than 50.}\label{fig:sc}
\end{center}
\end{figure}

As aforementioned, the kernel size should be made positively proportional to the tensor size, i.e., $k=\mathcal{O}(m)$. Thus, the above theorems suggest that the success of CNNM (and $\mathrm{DFT}_{\ell_1}$) requires $\rho_0 > 1 - \mathcal{O}(1/r_k(\mathbf{L}_0))$, which says that the sampling bound, the lower bound of $\rho_0$, has no direct link to the tensor dimension $m$. This is quite unlike the random sampling based matrix completion theories (e.g.,~\cite{Candes:2009:math,Chen:2015:tit}). To be more precise, consider the case of $k=m$ for simplicity. According to~\cite{Chen:2015:tit}, the sampling complexity required for recovering the convolution matrix, $\mathcal{A}_m(\mathbf{L}_0)\in\mathbb{R}^{m\times{}m}$, can be as low as $\mathcal{O}(r_m(\mathbf{L}_0)(\log{m})^2/m)$, which gives that the sampling bound should tend to decrease as $m$ grows. In fact, there is no conflict because the sampling regime under forecasting is deterministic rather than random. Figure~\ref{fig:sc} illustrates that the results derived from random sampling cannot apply to forecasting, confirming the certainty of our result. Even more, the sampling bound $1 - \mathcal{O}(1/r_k(\mathbf{L}_0))$ is pretty tight under the setup of forecasting, as we will show in Section~\ref{sec:exp:simu}.

In practice, the observed data is often a noisy version of $\mathcal{P}_{\Omega}(\mathbf{L}_0)$---or the target $\mathbf{L}_0$ is not strictly convolutionally low-rank as equal. In this case, one should relax the equality constraint and consider instead the following:
\begin{align}\label{eq:cnnm:noisy}
\min_{\mathbf{L}} \norm{\mathcal{A}_k(\mathbf{L})}_{*},\textrm{ s.t. }\|\mathcal{P}_{\Omega}(\mathbf{L} - \mathbf{M})\|_F\leq{}\epsilon,
\end{align}
where $\mathcal{P}_{\Omega}(\mathbf{M})$ denotes an observation of $\mathcal{P}_{\Omega}(\mathbf{L}_0)$, and $\epsilon\geq0$ is a parameter. The following theorem guarantees the recovery accuracy of~\eqref{eq:cnnm:noisy}.

\begin{theo}[Noisy]\label{main:thm:cnnm:noisy}
Use the same notations as in Theorem~\ref{main:thm:cnnm:noiseless}. Suppose that $\|\mathcal{P}_{\Omega}(\mathbf{M} - \mathbf{L}_0)\|_F\leq\epsilon$. If
\begin{align}\label{eq:samplebound1}
\rho_0 > 1 - \frac{0.22k}{\mu_k(\mathbf{L}_0)r_k(\mathbf{L}_0)m},
\end{align}
then any optimal solution $\mathbf{L}_o$ to the CNNM program~\eqref{eq:cnnm:noisy} gives a near recovery to the target tensor $\mathbf{L}_0$, in a sense that
\begin{align}\label{eq:errbound1}
\|\mathbf{L}_o-\mathbf{L}_0\|_F\leq{}(1+\sqrt{2})(38\sqrt{k}+2)\epsilon.
\end{align}
\end{theo}
Due to the error bound in~\eqref{eq:errbound1}, it seems that the recovery error produced by CNNM may increase with $\sqrt{k}$. Similar phenomena appear in many papers on matrix completion, e.g.,~\cite{CandesPIEEE}. This, however, is unlikely to be optimal and is indeed an unpleasant effect of the proof techniques commonly used in the community. As a result, it would be incorrect to speculate, based on the error bound in~\eqref{eq:errbound1}, that smaller kernel size $k$ always results in less recovery error. To boost the recovery performance of CNNM, in fact, $k$ should be chosen to minimize the sampling bound, as we have pointed out below Theorem~\ref{main:thm:cnnm:noiseless}. Once again, the relationship between CNNM and $\mathrm{DFT}_{\ell_1}$ leads to the following result:
\begin{coro}[Noisy]\label{main:coro:dft:noisy}
Use the same notations as in Theorem~\ref{main:thm:cnnm:noiseless}, and set $k_j=m_j,\forall{}j$. Suppose that $\|\mathcal{P}_{\Omega}(\mathbf{M} - \mathbf{L}_0)\|_F\leq\epsilon$ and $\mathbf{L}_o$ is an optimal solution to the following $\mathrm{DFT}_{\ell_1}$ program:
\begin{align}\label{eq:dft:noisy}
&\min_{\mathbf{L}} \norm{\mathcal{F}(\mathbf{L})}_1, \textrm{ s.t. } \|\mathcal{P}_{\Omega}(\mathbf{L}-\mathbf{M})\|_F \leq\epsilon.
\end{align}
If $\rho_0 > 1 - 0.22/(\mu_m(\mathbf{L}_0)\|\mathcal{F}(\mathbf{L}_0)\|_0)$, then $\mathbf{L}_o$ gives a near recovery to $\mathbf{L}_0$, in a sense that
\begin{align}
\|\mathbf{L}_o-\mathbf{L}_0\|_F\leq{}(1+\sqrt{2})(38\sqrt{m}+2)\epsilon.
\end{align}
\end{coro}

For any data tensor $\mathbf{M}$, one can always decompose it into $\mathbf{M} = \mathbf{L}_0 + \mathbf{N}$ with $\mathcal{A}_k(\mathbf{L}_0)$ being strictly low-rank and $\|\mathbf{N}\|_F\leq\epsilon$. So, the recovery error of CNNM is consistently bounded from above regardless the structure of data. Yet, this does not mean that CNNM can work well on all kinds of data: Whenever $\mathcal{A}_k(\mathbf{M})$ is far from being low-rank, the residual $\epsilon$ could be large and the obtained recovery is unnecessarily accurate.
\subsection{On Convolutional Low-Rankness}\label{sec:convlowrankness}
The analyses presented in the above subsection illustrate that, for CNNM to work well, it is important that the convolution matrix of the data tensor $\mathbf{M}$ is strictly or close to be low-rank, i.e., $\mathbf{M}$ itself is convolutionally low-rank or approximately so. This condition cannot be met by all kinds of data, but there do exist many examples of compliable. For example, suppose that $\mathbf{M}\in\mathbb{R}^{m_1\times\cdots\times{}m_n}$ is a periodic tensor with period $(\pi_1,\cdots,\pi_n)$ ($1\leq{}\pi_q\ll{}m_q,\forall{}q$); namely,
\begin{align}\label{eq:period}
\mathbf{M} = \mathcal{S}(\mathbf{M}, \pi_q, q), \forall{}q=1,\cdots,n,
\end{align}
where $\mathcal{S}(\cdot, \cdot, \cdot)$ is the multi-directional shift operator used in~\eqref{eq:multi-shift}. Then it may be easily seen that
\begin{align}
r_k(\mathbf{M})=\rank{\mathcal{A}_k(\mathbf{M})} \leq\prod_{q=1}^n\pi_q, \forall{}k.
\end{align}
That is, periodicity can lead to convolutional low-rankness in a strict manner. This result also implies that, while coping with periodic tensors, the kernel should be made as large as possible, i.e., $k_q=m_q$, $\forall{}q$.

Besides periodicity, the smoothness of a signal can also result in convolutional low-rankness in an approximate fashion. To see why, we shall begin with the case of $n=1$, i.e., $\mathbf{M}$ is a vector. In this case, the $j$th column of the convolution matrix $\mathcal{A}_k(\mathbf{M})$ is simply the vector obtained by circularly shifting the entries in $\mathbf{M}$ by $j-1$ positions. Intuitively, when $\mathbf{M}$ possesses substantial smoothness and the shift degree is relatively small, the signals before and after circular shift are mostly the same and therefore $\mathcal{A}_k(\mathbf{M})$ may possess low-rankness. To reach a rigorous conclusion, we consider the \emph{rank-$r$ approximation error} of matrices, which is denoted by $\varepsilon_r(\cdot)$ and defined as
\begin{align}
&\varepsilon_r(Y) = \min_{X}\|X-Y\|_F, \textrm{ s.t. } \rank{X}\leq{}r,
\end{align}
where $Y\in\mathbb{R}^{a\times{}b}$ is a matrix and $r$ is an integer between 1 and $\min(a,b)$. In the context of circular convolution, the smoothness of a vector $\mathbf{x}$ can be measured by a quantity, denoted as $\delta(\cdot)$, that is similar to the well-known \emph{total variation}:
\begin{align}\label{eq:tv}
\delta(\mathbf{x})=\|\mathbf{x}-\mathcal{S}(\mathbf{x})\|_2,
\end{align}
where $\|\cdot\|_2$ is the $\ell_2$ norm of a vector and $\mathcal{S}(\cdot)$ is the circular shift operator used in~\eqref{eq:circshift}. With these notations, it is provable that the rank-$r$ approximation error of the convolution matrix of $\mathbf{M}$ satisfies
\begin{align}\label{eq:apperr}
&\varepsilon_r(\mathcal{A}_k(\mathbf{M}))\leq\frac{\lceil\frac{k}{r}\rceil(k-r)}{2}\delta(\mathbf{M}),
\end{align}
where $\lceil\frac{k}{r}\rceil$ standards for the smallest integer greater than or equal to $k/r$. Hence, the smoothness, which has been adapted to the circulant boundary condition, provably leads to convolutional low-rankness in an approximate sense.
\begin{proof}When $\mathbf{M}$ is an $m$-dimensional vector, we have $\mathcal{A}_k(\mathbf{M})=[\mathbf{M},\mathcal{S}(\mathbf{M}),\cdots,\mathcal{S}^{k-1}(\mathbf{M})]$. Decompose $\mathcal{A}_k(\mathbf{M})$ into the concatenation of $r$ submatrices, namely $\mathcal{A}_k(\mathbf{M})=[A_1,A_2,\cdots,A_r]$, such that $A_i$ has $b_i$ columns with $1\leq{}b_i\leq\lceil\frac{k}{r}\rceil$ and $\sum_{i=1}^rb_i=k$. For $A_i\in\mathbb{R}^{m\times{}b_i}$, construct a rank-1 matrix $\hat{A}_i\in\mathbb{R}^{m\times{}b_i}$ by repeating the first column of $A_i$ for $b_i$ times. Then we have
\begin{align}
&\varepsilon_1(A_i)\leq\|A_i-\hat{A}_i\|_F\leq\sum_{c=0}^{b_i - 1} c\delta(\mathbf{M})\\\nonumber
&=\frac{b_i(b_i-1)}{2}\delta(\mathbf{M})\leq\frac{\lceil\frac{k}{r}\rceil(b_i-1)}{2}\delta(\mathbf{M}),
\end{align}
which gives that
\begin{align}
\varepsilon_r(\mathcal{A}_k(\mathbf{M}))\leq\sum_{i=1}^{r}\varepsilon_1(A_i)\leq\frac{\lceil\frac{k}{r}\rceil(k-r)}{2}\delta(\mathbf{M}).
\end{align}
\end{proof}
The above arguments can be easily extended to the general case of $n\geq1$. To prove that the conclusion in~\eqref{eq:apperr} holds for any tensors of order $n\geq1$, one just needs to generalize the definition in~\eqref{eq:tv} to the following:
\begin{align}\label{eq:tv2}
\delta(\mathbf{X})=\max_{1\leq{}q\leq{}n}\|\mathbf{X}-\mathcal{S}(\mathbf{X}, 1, q)\|_F,
\end{align}
where $\mathbf{X}$ is an order-$n$ tensor with $n\geq1$ and $\mathcal{S}(\cdot, \cdot, \cdot)$ is the multi-directional shift operator.

Since $\norm{\mathcal{F}(\mathbf{M})}_0=r_m(\mathbf{M}), \forall{}\mathbf{M}$, the result in~\eqref{eq:apperr} is also helpful for understanding the phenomenon of Fourier sparsity, which appears frequently in many domains, ranging from images~\cite{Don06} and videos~\cite{anan:silped:1996} to Boolean functions~\cite{OD08} and wideband channels~\cite{Lin2012}. More precisely, regarding the smooth signals with bounded values, e.g., images and videos, the smoothness quantity defined in~\eqref{eq:tv2} is often small, thereby convolutional low-rankness is very likely to occur, and so for Fourier sparsity.
\subsection{Optimization Algorithms}\label{sec:optalg}
\noindent\textbf{Algorithm for CNNM:} For the ease of implementation, we shall not try to solve problem~\eqref{eq:cnnm:noisy} directly, but instead consider its equivalent version as in the following:
\begin{align}\label{eq:cnnm:noisy2}
&\min_{\mathbf{L}} \norm{\mathcal{A}_k(\mathbf{L})}_{*}+\frac{\lambda{}k}{2}\|\mathcal{P}_{\Omega}(\mathbf{L}-\mathbf{M})\|_F^2,
\end{align}
where we amplify the parameter $\lambda$ by a factor of $k=\Pi_{j=1}^nk_j$ for the purpose of normalizing the two objectives to a similar scale. This problem is convex and can be solved by Alternating Direction Method of Multipliers (ADMM)~\cite{admm:1976,alm:2009:lin}. We first convert it to the following equivalent problem:
\begin{align}
\min_{\mathbf{L},Z} \norm{Z}_{*}+\frac{\lambda{}k}{2}\|\mathcal{P}_{\Omega}(\mathbf{L}-\mathbf{M})\|_F^2, \textrm{ s.t. } \mathcal{A}_k(\mathbf{L}) = Z.
\end{align}
Then the ADMM algorithm minimizes the augmented Lagrangian function,
\begin{align}
&\norm{Z}_{*}+\frac{\lambda{}k}{2}\|\mathcal{P}_{\Omega}(\mathbf{L}-\mathbf{M})\|_F^2+\langle\mathcal{A}_k(\mathbf{L}) - Z,Y\rangle \\\nonumber
&+ \frac{\tau}{2}\|\mathcal{A}_k(\mathbf{L}) - Z\|_F^2,
\end{align}
with respect to $\mathbf{L}$ and $Z$, respectively, by fixing the other variables and then updating the Lagrange multiplier $Y$ and the penalty parameter $\tau$. Namely, while fixing the other variables, the variable $Z$ is updated by
\begin{align}
Z = \arg\min_{Z}\frac{1}{\tau}\|Z\|_* + \frac{1}{2}\left\|Z - \left(\mathcal{A}_k(\mathbf{L})+\frac{Y}{\tau}\right)\right\|_F^2,
\end{align}
which is solved via Singular Value Thresholding (SVT)~\citep{svt:cai:2008}. While fixing the others, the variable $\mathbf{L}$ is updated via
\begin{align}
\mathbf{L} = (\lambda\mathcal{P}_{\Omega}+\tau\mathcal{I})^{-1}\left(\frac{\mathcal{A}_k^*(\tau{}Z - Y)}{k}+\lambda\mathcal{P}_{\Omega}(\mathbf{M})\right),
\end{align}
where the inverse operator is simply the entry-wise tensor division. The convergence of ADMM with two or fewer blocks has been well understood, and researchers had even developed advanced techniques to improve its convergence speed, see~\citep{admm:1976,lin:nips:2017}. While solving the CNNM problem, the computation of each ADMM iteration is dominated by the SVT step, which has a complexity of $\mathcal{O}(mk^2)$. Usually, depending on the increase rate of the penalty parameter $\tau$, the number of iterations for ADMM to get converged may range from tens to hundreds. \\

\noindent\textbf{Algorithm for $\mathbf{DFT_{\ell_1}}$:} We shall consider the following problem that is equivalent to the $\mathrm{DFT}_{\ell_1}$ program~\eqref{eq:dft:noisy}:
\begin{align}\label{eq:tc:dftl1}
&\min_{\mathbf{L}}\norm{\mathcal{F}(\mathbf{L})}_1+\frac{\lambda{}m}{2}\|\mathcal{P}_{\Omega}(\mathbf{L}-\mathbf{M})\|_F^2,
\end{align}
where $\lambda>0$ is a parameter. Due to the connection given in~\eqref{eq:dftnucconn}, the solution to~\eqref{eq:tc:dftl1} can be determined by finding a solution to~\eqref{eq:cnnm:noisy2} with $k_j=m_j,\forall{}j$. Yet, for the sake of efficiency, it is better to implement a specialized algorithm that can utilize the advantage of DFT.

Generally, the problem in~\eqref{eq:tc:dftl1} is solved in a similar way to the CNNM problem~\eqref{eq:cnnm:noisy2}. The main difference happens in updating $\mathbf{Z}$ and $\mathbf{L}$. While fixing the other variables and updating $\mathbf{Z}$, one needs to solve the following convex problem:
\begin{align}
\mathbf{Z} = \arg\min_{\mathbf{Z}}\frac{1}{\tau}\|\mathbf{Z}\|_1 + \frac{1}{2}\left\|\mathbf{Z} - \left(\mathcal{F}(\mathbf{L})+\frac{\mathbf{Y}}{\tau}\right)\right\|_F^2.
\end{align}
Note here that the variable $\mathbf{Z}$ is of complex-valued, and thus one needs to invoke Lemma 4.1 of~\cite{tpami_2013_lrr} to obtain a closed-form solution; namely,
\begin{align}
\mathbf{Z} = f_{1/\tau}\left(\mathcal{F}(\mathbf{L})+\frac{\mathbf{Y}}{\tau}\right),
\end{align}
where $f_{\alpha}(\cdot)$, a mapping parameterized by $\alpha>0$, is an entry-wise shrinkage operator given by
\begin{align}\label{eq:shrink}
f_{\alpha}(z) = \left\{\begin{array}{cc}
\frac{|z|-\alpha}{|z|}z, & \textrm{if } |z|>\alpha,\\
0, & \textrm{otherwise},
\end{array}\right.\quad\forall{}z\in\mathbb{C}.
\end{align}
While fixing the others, the variable $\mathbf{L}$ is updated via
\begin{align}
\mathbf{L} = (\lambda\mathcal{P}_{\Omega}+\tau\mathcal{I})^{-1}\left(\frac{\mathcal{F}^*(\tau{}\mathbf{Z} - \mathbf{Y})}{m}+\lambda\mathcal{P}_{\Omega}(\mathbf{M})\right),
\end{align}
where $\mathcal{F}^*$ denotes the Hermitian adjoint of DFT and is given by $m\mathcal{F}^{-1}$. As can be seen, the computational load is dominated by the calculations of the DFT operator as well as its inverse. Due to the strengths of the Fast Fourier Transform (FFT) algorithm~\cite{fft:1998}, calculating DFT or inverse DFT for an $m_1\times\cdots\times{}m_n$ tensor needs a computational complexity of only $\mathcal{O}(m\log{}m)$ with $m=\Pi_{j=1}^nm_j$.
\subsection{Discussions}\label{sc:main:diss}
\begin{figure}[h!]
\begin{center}
\includegraphics[width=0.48\textwidth]{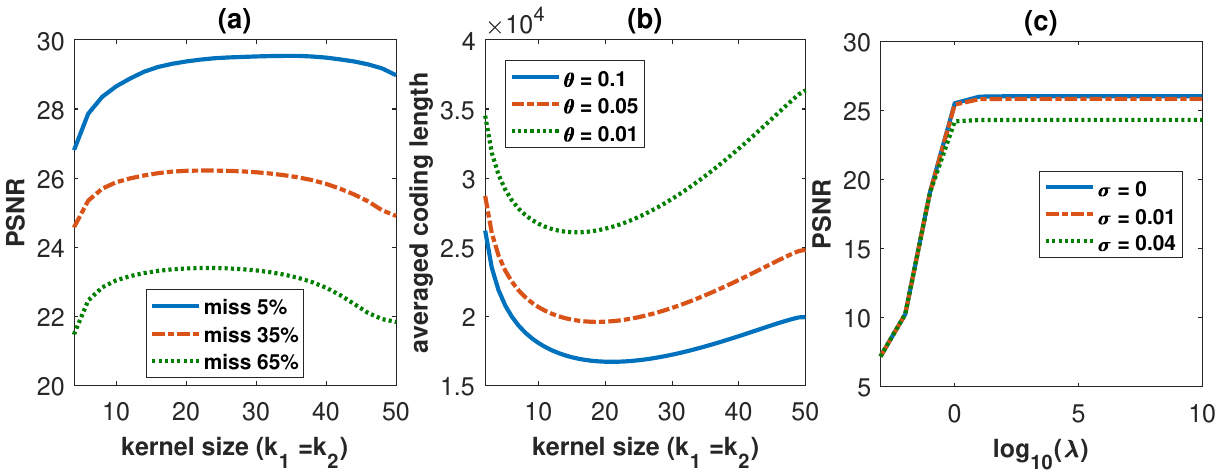}\vspace{-0.15in}
\caption{Exploring the influences of the parameters in CNNM, using a $50\times50$ image patch as the experimental data. (a) Plotting the recovery accuracy as a function of the kernel size. (b) Plotting the averaged coding length as a function of the kernel size. (c) Plotting the recovery accuracy as a function of the parameter $\lambda$. For the experiments in (c), the missing rate is set as 35\% and the observed entries are contaminated by iid Gaussian noise with mean 0 and standard deviation $\sigma$. Note that in this paper the PSNR measure is evaluated only on the missing entries. }\label{fig:paras}
\end{center}
\end{figure}
\noindent\textbf{On Influences of Parameters:} The hyper-parameters in CNNM mainly include the kernel size $k_1\times{}\cdots\times{}k_n$ and the regularization parameter $\lambda$.

According to the results in Figure~\ref{fig:sc}, it seems beneficial to use large kernels. But this is not the case with most real-world data. As shown in Figure~\ref{fig:paras}(a), the recovery accuracy of CNNM increases as the enlargement of the adopted kernel size at first, but then drops eventually as the kernel size continues to grow. In fact, both phenomena are consistent with the our theories, which say that the sampling bound is positively related to $r_k(\mathbf{L}_0)/k$ with $k=\Pi_{j=1}^nk_j$. For the particular example in Figure~\ref{fig:sc}, $r_k(\mathbf{L}_0)\equiv2$, $\forall{}k\geq2$, and thus large $k$ produces better recovery. However, on realistic data, the convolution rank may increase as the kernel size grows. To show the consistence in this case, we would like to investigate empirically the \emph{coding length}~\citep{yi:tpami:2007} of the convolution matrix of $\mathbf{L}_0\in\Re^{m_1\times\cdots\times{}m_n}$:
\begin{align}
&\mathrm{CL}_{\theta}(\mathcal{A}_k(\mathbf{L}_0)) \\\nonumber
&=\frac{1}{2}(m+k)\log\mathrm{det}\left(\Id +  \frac{m}{k\theta^2}\mathcal{A}_k(\mathbf{L}_0)(\mathcal{A}_k(\mathbf{L}_0))^T\right),
\end{align}
where $m=\Pi_{j=1}^nm_j$, $\mathrm{det}(\cdot)$ is the determinant of a matrix, and $\theta>0$ is a parameter. In general, $\mathrm{CL}_{\theta}(\mathcal{A}_k(\mathbf{L}_0))$ is no more than a computationally-friendly approximate to the convolution rank $r_k(\mathbf{L}_0)$, thereby a reasonable approximate to $r_k(\mathbf{L}_0)/k$ is given by
\begin{align}
&\mathrm{ACL}_{\theta}(\mathcal{A}_k(\mathbf{L}_0)) = \frac{\mathrm{CL}_{\theta}(\mathcal{A}_k(\mathbf{L}_0))}{k}\\\nonumber
&=\frac{1}{2}\left(\frac{m}{k}+1\right)\log\mathrm{det}\left(\Id + \frac{m}{k\theta^2}\mathcal{A}_k(\mathbf{L}_0)(\mathcal{A}_k(\mathbf{L}_0))^T\right),
\end{align}
where $\mathrm{ACL}_{\theta}(\cdot)$ is the \emph{averaged coding length} of a matrix.

As we can see from Figure~\ref{fig:paras}(b), the averaged coding length of $\mathcal{A}_k(\mathbf{L}_0)$ is minimized at some value between $k=1$ and $k=m$, and, interestingly, the minimizer can coincide with the point that maximizes the recovery accuracy provided that the parameter $\theta$ is chosen properly. So, to gain the ``best'' performance, the kernel size in CNNM should be set properly according to the structure of the target $\mathbf{L}_0$. Since $\mathbf{L}_0$ is unknown, estimating the parameters $\{k_1,\cdots{},k_n\}$ is essentially a challenging model selection problem, and there is no magic to ascertain the best choice. In fact, the uncertainty in determining $\{k_1,\cdots,k_n\}$ is related to \emph{interval forecasting} (see~\cite{lbcnnm:2021}). Regarding the setup of \emph{point forecasting} adopted in this paper, we would suggest some empirical rules. Let $k_j=\alpha_jm_j$, then the hyper-parameters $\{\alpha_1,\cdots,\alpha_n\}$ could be set as follows:
\begin{itemize}
\item[$\bullet$] The parameter $\alpha_n$, which is associated with the time-dimension, should satisfy $h/m_n<\alpha_n\leq1$, where $h$ is the forecast horizon. In most cases, $\alpha_n=0.5$ is a suitable choice.
\item[$\bullet$] The parameters $\{\alpha_1,\cdots{},\alpha_{n-1}\}$, which correspond to the non-time dimensions, just need to obey $0<\alpha_j\leq1, \forall{}1\leq{j}\leq{}n-1$. Empirically, while handling natural images and videos, near optimal recovery performance is often attained at $\alpha_j=0.25,\forall{}j$.
\end{itemize}

Figure~\ref{fig:paras}(c) shows the influence of the parameter $\lambda$. As we can see, there seems no need to tune $\lambda$ carefully. The reason is probably because the nuclear norm is already good at handling heavy-tailed data. So, no matter whether the observations are contaminated by noise or not, we would suggest setting $\lambda=1000$ for CNNM and $\mathrm{DFT}_{\ell_1}$.\\

\noindent\textbf{Is Convolution Rank Minimization NP-Hard?} Consider the original form of the CNNM problem~\eqref{eq:tc:cnnm:exact}:
\begin{align}\label{eq:convrankmin}
\min_{\mathbf{L}} \rank{\mathcal{A}_k(\mathbf{L})}, \textrm{ s.t. } \mathcal{P}_{\Omega}(\mathbf{L}-\mathbf{L}_0) = 0.
\end{align}
It seems not easy to figure out whether the above problem is NP-hard, as it is seemingly difficult to be reduced from some existing NP-hard problem. So we shall examine instead the special case of $k_i=m_i$, $1\leq{}i\leq{}n$, i.e., the original form of the $\mathrm{DFT}_{\ell_1}$ problem in~\eqref{eq:tc:dftl1:exact}:
\begin{align}\label{eq:convrankmin1}
\min_{\mathbf{L}} \|\mathcal{F}(\mathbf{L})\|_0, \textrm{ s.t. } \mathcal{P}_{\Omega}(\mathbf{L}-\mathbf{L}_0) = 0.
\end{align}
According to the derivations around~\eqref{eq:dftl1:equi}, the above $\ell_0$ minimization problem is equivalent to the classic problem of finding the sparsest vector in an affine subspace, which is known to be NP-hard~\cite{siam:1995:bk}. As a consequence, in general cases, the problem in~\eqref{eq:convrankmin1} is NP-hard, and so for~\eqref{eq:convrankmin}.

It is worth mentioning that the convolution nuclear norm might not be the tightest convex approximation of the convolution rank. This is because, on the set of convolution matrices, the matrix nuclear norm is probably no longer the convex envelope of the rank function. That said, our theorems prove that the convolution nuclear norm minimization can produce exact solutions, as long as certain sampling complexity conditions are met.
\section{Mathematical Proofs}\label{sec:proof}
This section presents in detail the proofs to the proposed lemmas and theorems.
\subsection{Proof to Lemma~\ref{lem:alge:conv}}\label{sec:proof1}
\begin{proof}
We shall revisit the notations defined in Section~\ref{sec:convmtx}. Denote by $\mathcal{S}$ the ``circshift'' operator in Matlab; namely, $\mathcal{S}(\mathbf{G},u,v)$ circularly shifts the elements in tensor $\mathbf{G}$ by $u$ positions along the $v$th direction. For any $(i_1,\cdots,i_n)\in\{1,\cdots,k_1\}\times\cdots\times\{1,\cdots,k_n\}$, we define an invertible operator $\mathcal{T}_{(i_1,\cdots,i_n)}: \mathbb{R}^{m_1\times\cdots\times{}m_n}\rightarrow\mathbb{R}^m$ as
\begin{align}
 \mathcal{T}_{(i_1,\cdots,i_n)}(\mathbf{G}) = \mathrm{vec}(\mathbf{G}_n),\forall{}\mathbf{G}\in\mathbb{R}^{m_1\times\cdots\times{}m_n},
\end{align}
where $\mathrm{vec}(\cdot)$ is the vectorization operator and $\mathbf{G}_n$ is determined by the following recursion rule:
\begin{align}
\mathbf{G}_0= \mathbf{G},\mathbf{G}_q=\mathcal{S}(\mathbf{G}_{q-1},i_q-1,q), 1\leq{}q\leq{}n.
\end{align}
Suppose that $j=1+\sum_{a=1}^{n}(i_a-1)\Pi_{b=0}^{a-1}k_b$, where it is assumed conveniently that $k_0 = 1$. Then we have
\begin{align}\label{eq:temp:1}
[\mathcal{A}_k(\mathbf{X})]_{:,j} = \mathcal{T}_{(i_1,\cdots,i_n)}(\mathbf{X}).
\end{align}
According to the definition of the Hermitian adjoint operator given in~\eqref{eq:adjoint}, we have
\begin{align}\label{eq:temp:2}
\mathcal{A}_k^*(Z) = \sum_{i_1,\cdots,i_n}\mathcal{T}_{(i_1,\cdots,i_n)}^{-1}([Z]_{:,j}), \forall{}Z\in\mathbb{R}^{m\times{}k},
\end{align}
where it is worth noting that the number $j$ functionally depends on the index $(i_1,\cdots,i_n)$. By~\eqref{eq:temp:1} and~\eqref{eq:temp:2},
\begin{align}
\mathcal{A}_k^*\mathcal{A}_k(\mathbf{X})=\sum_{i_1,\cdots,i_n}\mathcal{T}_{(i_1,\cdots,i_n)}^{-1}\mathcal{T}_{(i_1,\cdots,i_n)}(\mathbf{X})=k\mathbf{X}.
\end{align}
The second claim is easy to prove. By~\eqref{eq:omega},~\eqref{eq:conv:omega} and~\eqref{eq:temp:1},
\begin{align}
&[\mathcal{A}_k\mathcal{P}_{\Omega}(\mathbf{X})]_{:,j} =\mathcal{T}_{(i_1,\cdots,i_n)}(\mathbf{\Theta}_{\Omega}\circ{}\mathbf{X})=\\ \nonumber &\mathcal{T}_{(i_1,\cdots,i_n)}(\mathbf{\Theta}_{\Omega})\circ{}\mathcal{T}_{(i_1,\cdots,i_n)}(\mathbf{X})=[\Theta_{\Omega_\mathcal{A}}]_{:,j}\circ{}[\mathcal{A}_k(\mathbf{X})]_{:,j}\\\nonumber
&=[\mathcal{P}_{\Omega_{\mathcal{A}}}\mathcal{A}_k(\mathbf{X})]_{:,j}.
\end{align}
It remains to prove the third claim. By~\eqref{eq:omega} and~\eqref{eq:conv:omega},
\begin{align}
&\mathcal{A}_k^*\mathcal{P}_{\Omega_{\mathcal{A}}}(Y) =\mathcal{A}_k^*(\Theta_{\Omega_{\mathcal{A}}}\circ{}Y) = \mathcal{A}_k^*(\mathcal{A}_k(\mathbf{\Theta}_{\Omega})\circ{}Y),
\end{align}
which, together with~\eqref{eq:temp:1} and~\eqref{eq:temp:2}, gives that
\begin{align}
&\mathcal{A}_k^*\mathcal{P}_{\Omega_{\mathcal{A}}}(Y)=\sum_{i_1,\cdots,i_n}\mathcal{T}_{(i_1,\cdots,i_n)}^{-1}([\mathcal{A}_k(\mathbf{\Theta}_\Omega)\circ{}Y]_{:,j})\\\nonumber
&=\sum_{i_1,\cdots,i_n}\mathcal{T}_{(i_1,\cdots,i_n)}^{-1}([\mathcal{A}_k(\mathbf{\Theta}_\Omega)]_{:,j})\circ{}\mathcal{T}_{(i_1,\cdots,i_n)}^{-1}([Y]_{:,j})\\\nonumber
&=\sum_{i_1,\cdots,i_n}\mathbf{\Theta}_{\Omega}\circ{}\mathcal{T}_{(i_1,\cdots,i_n)}^{-1}([Y]_{:,j})= \mathcal{P}_{\Omega}\mathcal{A}_k^*(Y).
\end{align}
\end{proof}
\subsection{Proof to Theorem~\ref{main:thm:cnnm:noiseless}}
The proof process is quite standard. We shall first prove the following lemma that establishes the conditions under which the solution to~\eqref{eq:tc:cnnm:exact} is unique and exact.
\begin{lemm}\label{lem:dual}Suppose the skinny SVD of the convolution matrix of $\mathbf{L}_0$ is given by $\mathcal{A}_k(\mathbf{L}_0)=U_0\Sigma_0V_0^T$. Denote by $\mathcal{P}_{T_0}(\cdot)=$ $U_0U_0^T(\cdot)+(\cdot)V_0V_0^T-U_0U_0^T(\cdot)V_0V_0^T$ the orthogonal projection onto the sum of $U_0$ and $V_0$. Then $\mathbf{L}_0$ is the unique minimizer to the problem in~\eqref{eq:tc:cnnm:exact} provided that:
\begin{itemize}
\item[1.] $\mathcal{P}_{\Omega_{\mathcal{A}}}^\bot\cap\mathcal{}P_{T_0}=\{0\}$.
\item[2.] There exists $Y\in\mathbb{R}^{m\times{}k}$ such that $\mathcal{}P_{T_0}\mathcal{P}_{\Omega_{\mathcal{A}}}(Y) = U_0V_0^T$ and $\|\mathcal{}P_{T_0}^\bot\mathcal{P}_{\Omega_{\mathcal{A}}}(Y)\|<1$.
\end{itemize}
\end{lemm}
\begin{proof}
Take $W = \mathcal{}P_{T_0}^\bot\mathcal{P}_{\Omega_{\mathcal{A}}}(Y)$. Then $\mathcal{A}_k^*(U_0V_0^T+W) = \mathcal{A}_k^*\mathcal{P}_{\Omega_{\mathcal{A}}}(Y)$. By Lemma~\ref{lem:alge:conv},
\begin{align}
\mathcal{A}_k^*\mathcal{P}_{\Omega_{\mathcal{A}}}(Y) = \mathcal{P}_{\Omega}\mathcal{A}_k^*(Y)\in\mathcal{P}_{\Omega}.
\end{align}
By the standard convexity arguments shown in~\citep{book:convex}, $\mathbf{L}_0$ is an optimal solution to the convex optimization problem in~\eqref{eq:tc:cnnm:exact}.

It remains to prove that $\mathbf{L}_0$ is the unique minimizer. To do this, we consider a feasible solution $\mathbf{L}_0+\mathbf{\Delta}$ with $\mathcal{P}_{\Omega}(\mathbf{\Delta}) = 0$, and we shall show that the objective value strictly increases unless $\mathbf{\Delta}=0$. Due to the convexity of the convolution nuclear norm, we have
\begin{align}
&\|\mathcal{A}_k(\mathbf{L}_0+\mathbf{\Delta})\|_* - \|\mathcal{A}_k(\mathbf{L}_0)\|_* \geq \langle{}\mathcal{A}_k^*(U_0V_0^T+H), \mathbf{\Delta}\rangle \\\nonumber
&= \langle{}U_0V_0^T+H, \mathcal{A}_k(\mathbf{\Delta})\rangle,
\end{align}
where $H\in\mathcal{P}_{T_0}^\bot$ and $\|H\|\leq{}1$. The first inequality above follows from a basic property of convexity. Namely, if $f(\cdot)$ is a convex function, we have $f(\mathbf{L}_0+\mathbf{\Delta})-f(\mathbf{L}_0)\geq\langle\partial_{\mathbf{L}_0}f,\mathbf{\Delta}\rangle$, where $\partial_{\mathbf{L}_0}f$ is the subgradient of $f$ at $\mathbf{L}_0$. For $f(\mathbf{L}) = \|\mathcal{A}_k(\mathbf{L})\|_*$, we have $\partial_{\mathbf{L}_0}f=\mathcal{A}_k^*(U_0V_0^T+H)$, where $U_0V_0^T+H$ is the subgradient of the nuclear norm at $\mathcal{A}_k(\mathbf{L}_0)$.

By the duality between the operator and nuclear norms, we can always choose an $H$ such that
\begin{align}
\langle{}H, \mathcal{A}_k(\mathbf{\Delta})\rangle = \|\mathcal{P}_{T_0}^\bot\mathcal{A}_k(\mathbf{\Delta})\|_*.
\end{align}
In addition, it follows from Lemma~\ref{lem:alge:conv} that
\begin{align}
 &\langle{}\mathcal{P}_{\Omega_{\mathcal{A}}}(Y), \mathcal{A}_k(\mathbf{\Delta})\rangle =  \langle{}Y, \mathcal{P}_{\Omega_{\mathcal{A}}}\mathcal{A}_k(\mathbf{\Delta})\rangle \\\nonumber
 &= \langle{}Y, \mathcal{A}_k\mathcal{P}_{\Omega}(\mathbf{\Delta})\rangle = 0.
\end{align}
Hence, we have
\begin{align}
&\langle{}U_0V_0^T+H, \mathcal{A}_k(\mathbf{\Delta})\rangle = \langle{}\mathcal{P}_{\Omega_{\mathcal{A}}}(Y)+H - W, \mathcal{A}_k(\mathbf{\Delta})\rangle \\\nonumber
&= \langle{}H - W, \mathcal{A}_k(\mathbf{\Delta})\rangle\geq(1-\|W\|)\|\mathcal{P}_{T_0}^\bot\mathcal{A}_k(\mathbf{\Delta})\|_*.
\end{align}
Since $\|W\|<1$, $\|\mathcal{A}_k(\mathbf{L}_0+\mathbf{\Delta})\|_*$ is greater than $\|\mathcal{A}_k(\mathbf{L}_0)\|_*$ unless $\mathcal{A}_k(\mathbf{\Delta})\in\mathcal{P}_{T_0}$. Note that $\mathcal{P}_{\Omega_{\mathcal{A}}}\mathcal{A}_k(\mathbf{\Delta}) = \mathcal{A}_k\mathcal{P}_{\Omega}(\mathbf{\Delta}) = 0$, i.e., $\mathcal{A}_k(\mathbf{\Delta})\in\mathcal{P}_{\Omega_{\mathcal{A}}}^\bot$. Since $\mathcal{P}_{\Omega_{\mathcal{A}}}^\bot\cap\mathcal{}P_{T_0}=\{0\}$, it follows that $\mathcal{A}_k(\mathbf{\Delta})=0$, which immediately leads to $\mathbf{\Delta}=0$.
\end{proof}

In the rest of the proof, we shall show how we will prove the dual conditions listed in Lemma~\ref{lem:dual}. Notice that, even if the locations of the missing entries are arbitrarily distributed, each column of $\Omega_{\mathcal{A}}$ has exactly a cardinality of $\rho_0m$, and each row of $\Omega_{\mathcal{A}}$ contains at least $k - (1-\rho_0)m$ elements. Denote by $\rho$ the smallest fraction of observed entries in each row and column of $\mathcal{A}_k(\mathbf{L}_0)$. Provided that $\rho_0>1-0.25k/(\mu_k(\mathbf{L}_0)r_k(\mathbf{L}_0)m)$, we have
\begin{align}
\rho\geq{}\frac{k-(1-\rho_0)m}{k}>1-\frac{0.25}{\mu_k(\mathbf{L}_0)r_k(\mathbf{L}_0)}.
\end{align}
Then it follows from Lemma~\ref{lem:iso:rcn} that $\mathcal{A}_k(\mathbf{L}_0)$ is $\Omega_\mathcal{A}/\Omega_\mathcal{A}^T$-isomeric and $\gamma_{\Omega_\mathcal{A},\Omega_\mathcal{A}^T}(\mathcal{A}_k(\mathbf{L}_0))>0.75$. Thus, according to Lemma 5.11 of~\citep{liu:tpami:2021}, we have
\begin{align}
\|\mathcal{P}_{T_0}\mathcal{P}_{\Omega_{\mathcal{A}}}^\bot\mathcal{P}_{T_0}\|\leq2(1-\gamma_{\Omega_\mathcal{A},\Omega_\mathcal{A}^T}(\mathcal{A}_k(\mathbf{L}_0)))<0.5<1,
\end{align}
which, together with Lemma 5.6 of~\citep{liu:tpami:2021}, results in $\mathcal{P}_{\Omega_{\mathcal{A}}}^\bot\cap\mathcal{}P_{T_0}=\{0\}$.  As a consequence, we could define $Y$ as
\begin{align}
Y = \mathcal{P}_{\Omega_{\mathcal{A}}}\mathcal{P}_{T_0}(\mathcal{P}_{T_0}\mathcal{P}_{\Omega_{\mathcal{A}}}\mathcal{P}_{T_0})^{-1}(U_0V_0^T).
\end{align}
It can be verified that $\mathcal{}P_{T_0}\mathcal{P}_{\Omega_{\mathcal{A}}}(Y) = U_0V_0^T$. Moreover, it follows from Lemma 5.12 of~\citep{liu:tpami:2021} that
\begin{align}
&\|\mathcal{}P_{T_0}^\bot\mathcal{P}_{\Omega_{\mathcal{A}}}(Y)\| \leq \|\mathcal{P}_{T_0}^\bot\mathcal{P}_{\Omega_{\mathcal{A}}}\mathcal{P}_{T_0}(\mathcal{P}_{T_0}\mathcal{P}_{\Omega_{\mathcal{A}}}\mathcal{P}_{T_0})^{-1}\|\|U_0V_0^T\|\\\nonumber
 &= \sqrt{\frac{1}{1-\|\mathcal{P}_{T_0}\mathcal{P}_{\Omega_{\mathcal{A}}}^\bot\mathcal{P}_{T_0}\|}-1}<1,
\end{align}
which finishes to construct the dual certificate.

\subsection{Proof to Theorem~\ref{main:thm:cnnm:noisy}}
\begin{proof}
Let $\mathbf{N} = \mathbf{L}_o - \mathbf{L}_0$ and denote $N_{\mathcal{A}} =\mathcal{A}_k(\mathbf{N})$. Notice that $\|\mathcal{P}_{\Omega}(\mathbf{L}_o-\mathbf{M})\|_F\leq\epsilon$ and $\|\mathcal{P}_{\Omega}(\mathbf{M} - \mathbf{L}_0)\|_F\leq\epsilon$. By triangle inequality, $\|\mathcal{P}_{\Omega}(\mathbf{N})\|_F\leq2\epsilon$. Thus,
\begin{align}
\|\mathcal{P}_{\Omega_{\mathcal{A}}}(N_{\mathcal{A}})\|_F^2 = \|\mathcal{A}_k\mathcal{P}_{\Omega}(\mathbf{N})\|_F^2= k\|\mathcal{P}_{\Omega}(\mathbf{N})\|_F^2\leq4k\epsilon^2.
\end{align}
To bound $\|\mathbf{N}\|_F$, it is sufficient to bound $\|N_{\mathcal{A}}\|_F$. So, it remains to bound $\|\mathcal{P}_{\Omega_{\mathcal{A}}}^\bot(N_{\mathcal{A}})\|_F$. To do this, we define $Y$ and $W$ in the same way as in the proof to Theorem~\ref{main:thm:cnnm:noiseless}. Since $\mathbf{L}_o=\mathbf{L}_0+\mathbf{N}$ is an optimal solution to~\eqref{eq:cnnm:noisy}, we have the following:
\begin{align}
&0\geq\|\mathcal{A}_k(\mathbf{L}_0+\mathbf{N})\|_* - \|\mathcal{A}_k(\mathbf{L}_0)\|_*\\\nonumber
&\geq(1-\|W\|)\|\mathcal{P}_{T_0}^\bot(N_{\mathcal{A}})\|_* + \langle{}\mathcal{P}_{\Omega_{\mathcal{A}}}(Y),N_{\mathcal{A}}\rangle.
\end{align}
Provided that $\rho_0>1-0.22/(\mu(\mathbf{L}_0)r(\mathbf{L}_0))$, we can prove that $\|W\|=\|\mathcal{P}_{T_0}^\bot\mathcal{P}_{\Omega_{\mathcal{A}}}(Y)\|<0.9$. As a consequence, we have the following:
\begin{align}
&\|\mathcal{P}_{T_0}^\bot(N_{\mathcal{A}})\|_*\leq-10\langle{}\mathcal{P}_{\Omega_{\mathcal{A}}}(Y),\mathcal{P}_{\Omega_{\mathcal{A}}}(N_{\mathcal{A}})\rangle\\\nonumber
&\leq10\|\mathcal{P}_{\Omega_{\mathcal{A}}}(Y)\|\|\mathcal{P}_{\Omega_{\mathcal{A}}}(N_{\mathcal{A}})\|_*\leq19\|\mathcal{P}_{\Omega_{\mathcal{A}}}(N_{\mathcal{A}})\|_*\\\nonumber
&\leq19\sqrt{k}\|\mathcal{P}_{\Omega_{\mathcal{A}}}(N_{\mathcal{A}})\|_F\leq38k\epsilon,
\end{align}
from which it follows that $\|\mathcal{P}_{T_0}^\bot(N_{\mathcal{A}})\|_F\leq\|\mathcal{P}_{T_0}^\bot(N_{\mathcal{A}})\|_*\leq38k\epsilon$, and which simply leads to
\begin{align}
\|\mathcal{P}_{T_0}^\bot\mathcal{P}_{\Omega_{\mathcal{A}}}^\bot(N_{\mathcal{A}})\|_F\leq(38k+2\sqrt{k})\epsilon.
\end{align}
We also have
\begin{align}
&\|\mathcal{P}_{\Omega_{\mathcal{A}}}\mathcal{P}_{T_0}\mathcal{P}_{\Omega_{\mathcal{A}}}^\bot(N_{\mathcal{A}})\|_F^2\\\nonumber
&=\langle\mathcal{P}_{T_0}\mathcal{P}_{\Omega_{\mathcal{A}}}\mathcal{P}_{T_0}\mathcal{P}_{\Omega_{\mathcal{A}}}^\bot(N_{\mathcal{A}}),\mathcal{P}_{T_0}\mathcal{P}_{\Omega_{\mathcal{A}}}^\bot(N_{\mathcal{A}})\rangle\\\nonumber
&\geq(1-\|\mathcal{P}_{T_0}\mathcal{P}_{\Omega_{\mathcal{A}}}^\bot\mathcal{P}_{T_0}\|)\|\mathcal{P}_{T_0}\mathcal{P}_{\Omega_{\mathcal{A}}}^\bot(N_{\mathcal{A}})\|_F^2,\\\nonumber
&\geq\frac{1}{2}\|\mathcal{P}_{T_0}\mathcal{P}_{\Omega_{\mathcal{A}}}^\bot(N_{\mathcal{A}})\|_F^2,
\end{align}
which gives that
\begin{align}
&\|\mathcal{P}_{T_0}\mathcal{P}_{\Omega_{\mathcal{A}}}^\bot(N_{\mathcal{A}})\|_F^2\leq2\|\mathcal{P}_{\Omega_{\mathcal{A}}}\mathcal{P}_{T_0}\mathcal{P}_{\Omega_{\mathcal{A}}}^\bot(N_{\mathcal{A}})\|_F^2\\\nonumber
&=2\|\mathcal{P}_{\Omega_{\mathcal{A}}}\mathcal{P}_{T_0}^\bot\mathcal{P}_{\Omega_{\mathcal{A}}}^\bot(N_{\mathcal{A}})\|_F^2\leq2(38k+2\sqrt{k})^2\epsilon^2.
\end{align}
Combining the above justifications, we have
\begin{align}
&\|N_{\mathcal{A}}\|_F\leq\|\mathcal{P}_{T_0}\mathcal{P}_{\Omega_{\mathcal{A}}}^\bot(N_{\mathcal{A}})\|_F + \|\mathcal{P}_{T_0}\mathcal{P}_{\Omega_{\mathcal{A}}}(N_{\mathcal{A}})\|_F\\\nonumber
&+ \|\mathcal{P}_{T_0}^\bot(N_{\mathcal{A}})\|_F\leq(\sqrt{2}+1)(38k + 2\sqrt{k})\epsilon.
\end{align}
Finally, the fact $\|\mathbf{N}\|_F = \|N_{\mathcal{A}}\|_F/\sqrt{k}$ finishes the proof.
\end{proof}
\begin{figure}[h!]
\begin{center}
\includegraphics[width=0.48\textwidth]{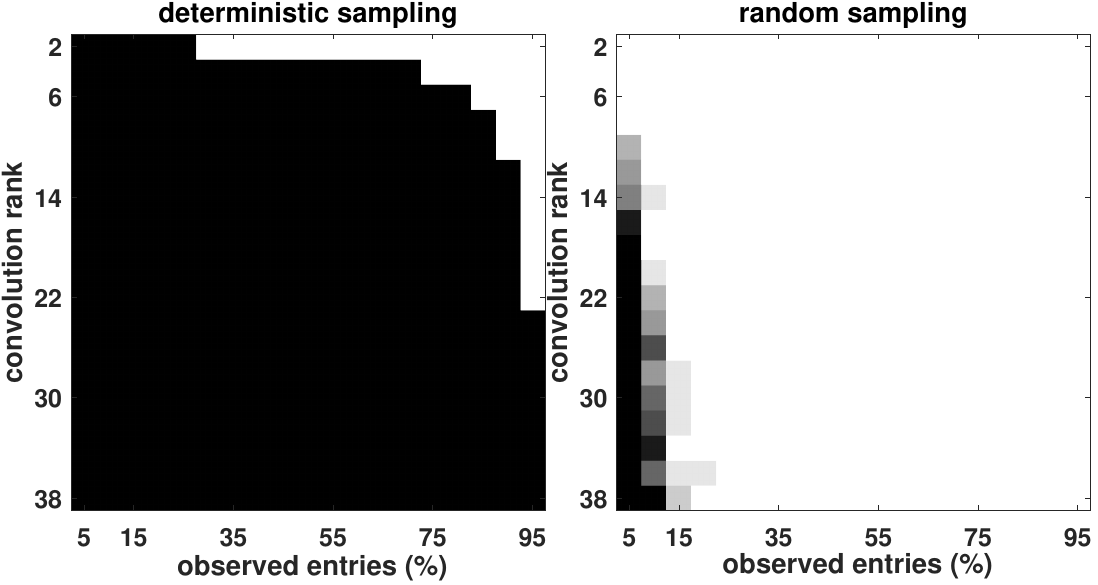}\vspace{-0.1in}
\caption{Investigating the recovery performance of $\mathrm{DFT}_{\ell_1}$ under the setup of random sampling and forecasting (i.e., deterministic sampling). The white and black areas mean ``success'' and ``failure'', respectively, where the recovery is regarded as being successful iff $\mathrm{PSNR}>50$.}\label{fig:transit}
\end{center}
\end{figure}
\begin{figure*}
\begin{center}
\includegraphics[width=0.95\textwidth]{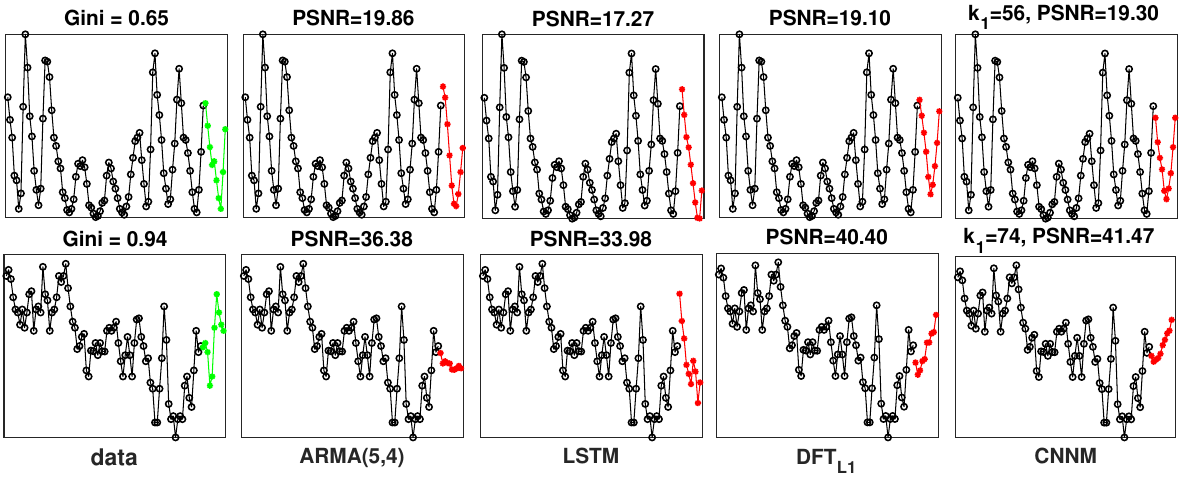}\vspace{-0.15in}
\caption{Results for univariate time series forecasting. The left figure shows the entire series used for experiments, where the observed and future entries are plotted with black and green markers, respectively. The sparsity degree of the Fourier transform of a series is measured by Gini~\cite{gini:tit:2009}.}\label{fig:1dts}
\end{center}
\end{figure*}
\section{Experiments}\label{sec:exp}
All experiments are conducted on the Matlab 2019a platform. The source codes are available at https://github.com /gcliu1982/CNNM.
\subsection{Simulations}\label{sec:exp:simu}
We shall experiment with synthetic data to verify the theorems proven in this paper. We generate $\mathbf{M}$ and $\mathbf{L}_0$ according to a model as follows: $[\mathbf{M}]_t = [\mathbf{L}_0]_t = \sum_{i=1}^a\sin(2ti\pi/m)$, where $m=1000$, $t=1,\cdots,$ $m$ and $a=1,\cdots,19$. So the target $\mathbf{L}_0$ is a univariate time series of dimension 1000, with $\|\mathcal{F}(\mathbf{L}_0)\|_0=2a$ and $\mu_m(\mathbf{L}_0)=1$. The values in $\mathbf{L}_0$ are further normalized to have a maximum of 1. Regarding the locations of observed entries, we consider two settings: One is to randomly select a subset of entries as in Figure~\ref{fig:sc}(c), referred to as ``random sampling'', the other is the forecasting setup used in Figure~\ref{fig:sc}(b), referred to as ``deterministic sampling''. The observation fraction is set as $\rho_0=0.05,0.1,\cdots,0.95$. For the setup of random sampling, we run 20 trials, and thus there are $19\times19\times21=7581$ simulations in total.

The results are shown in Figure~\ref{fig:transit}, the left part of which illustrates that the sampling complexity proven in Corollary~\ref{main:coro:dft:noiseless} is indeed tight: The boundary between success and failure is quite consistent with the curve given by $1-\mathcal{O}(1/\norm{\mathcal{F}(\mathbf{L}_0)}_0)$. Also, as we can see, the recovery of future observations is much more challenging than restoring the missing entries chosen uniformly at random. In fact, the sampling pattern of forecasting is almost the \emph{worst case} of arbitrary sampling. In these experiments, the convolution rank $r_k(\mathbf{L}_0)$ does not grow with the kernel size, so there is no benefit to try different kernel sizes in CNNM.
\subsection{Univariate Time Series Forecasting}\label{sec:exp:uni}
Now we consider two real-world time series downloaded from Time Series Data Library (TSDL): One for the annual Wolfer sunspot numbers from 1770 to 1869 with $m=100$, and the other for the highest mean monthly levels of Lake Michigan from 1860 to 1955 with $m=96$. We consider for comparison the well-known methods of Auto-Regressive Moving Average (ARMA) and Long Short-Term Memory (LSTM). ARMA contains many hyper-parameters, which are manually tuned to maximize its recovery accuracy (in terms of PSNR) on the first sequence. The LSTM architecture used for experiment is consist of four layers, including an input layer with 1 unit, a hidden LSTM layer with 200 units, a fully connected layer and a regression layer.

The results are shown in Figure~\ref{fig:1dts}. Via manually choosing the best parameters, ARMA can achieve the best performance on the first series. But its results are unsatisfactory while applying the same parametric setting to the second one. By contrast, $\mathrm{DFT}_{\ell_1}$ produces reasonable forecasts on both series that differs greatly in the evolution rules. This is not incredible, because the method never assumes explicitly how the future entries are related to the previously observed samples, and it is indeed the Fourier sparsity of the series itself that enables the recovery of the unseen future data. Via choosing proper kernel sizes, CNNM can future outperform $\mathrm{DFT}_{\ell_1}$---though the improvements here are mild. A more comprehensive evaluation about the forecasting performance of CNNM can be found in~\cite{lbcnnm:2021}.
\begin{figure}[h!]
\begin{center}
\includegraphics[width=0.35\textwidth]{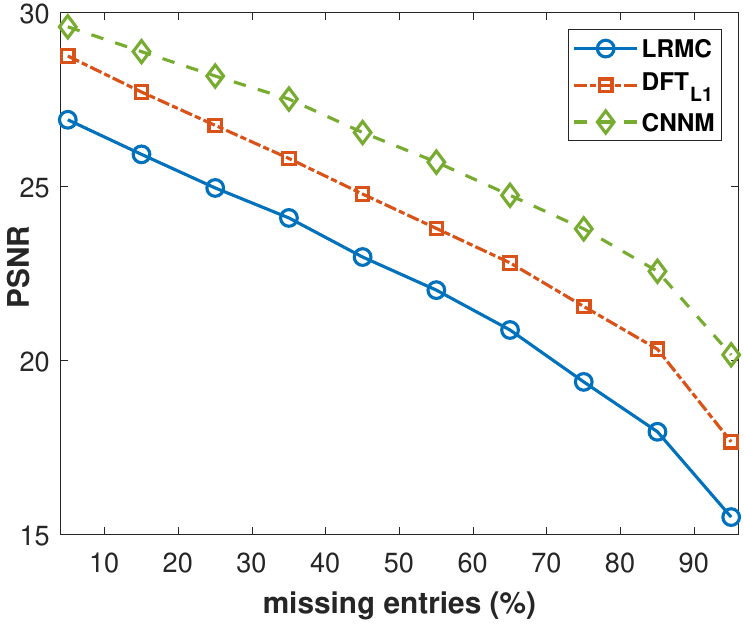}\vspace{-0.15in}
\caption{Evaluating the recovery performance of various methods, using the $200\times200$ Boats image as the experimental data. The missing entries are chosen uniformly at random, and the numbers plotted above are averaged from 20 random trials.}\label{fig:lena:acc}\vspace{-0.1in}
\end{center}
\end{figure}
\begin{table}[h!]
\caption{PSNR and running time on the $200\times200$ Boats image. In these experiments, 60\% of the image pixels are randomly chosen to be missing. The numbers shown below are collected from 20 runs.}\label{tb:lena}\vspace{-0.2in}
\begin{center}
\begin{tabular}{|c|c|c|}\hline
    Methods                  &PSNR            & Time (seconds)\\\hline
    LRMC                     &$21.43\pm0.11$  & $15.9\pm0.3$\\
    $\mathrm{DFT}_{\ell_1}$  & $23.25\pm0.10$ &$\mathbf{0.49\pm0.06}$\\
    CNNM($13\times13$)       &  $24.92\pm0.11$ &$33.2\pm0.4$\\
     CNNM($23\times23$)       & $25.15\pm0.12$ &160$\pm1$\\
     CNNM($33\times33$)       & $25.23\pm0.12$ &$429\pm2$\\
     CNNM($43\times43$)       & $25.27\pm0.12$ &$968\pm10$\\
     CNNM($53\times53$)       &  $\mathbf{25.28\pm0.12}$  &$2255\pm19$\\
     CNNM($63\times63$)       &  $25.26\pm0.12$  &$4925\pm42$\\\hline
\end{tabular}
\end{center}
\end{table}
\subsection{Image Completion}
\begin{figure*}
\begin{center}
\includegraphics[width=0.98\textwidth]{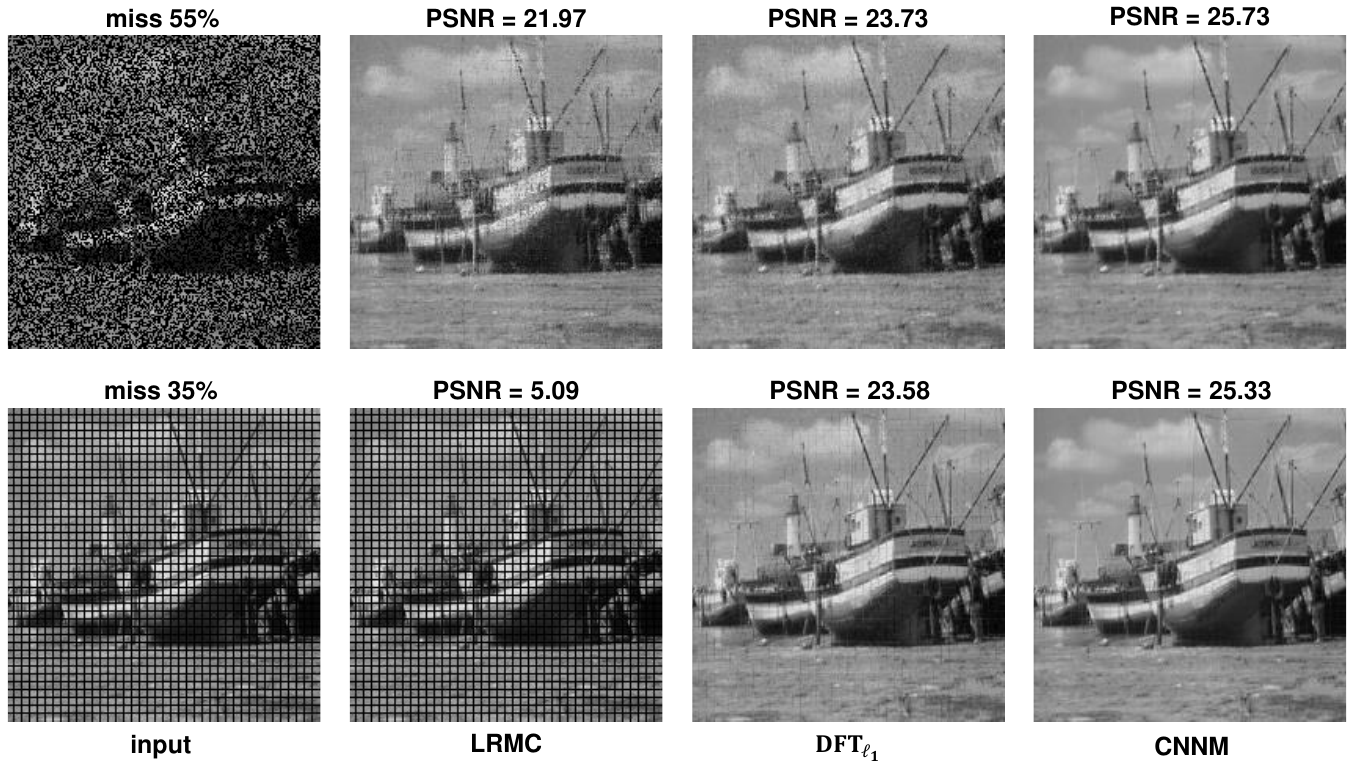}\vspace{-0.15in}
\caption{Examples of image completion. From left to right: the input data, the results by LRMC, the results by $\mathrm{DFT}_{\ell_1}$, and the results by CNNM. The Gini index of the Fourier transform of the original Boats image is 0.63. Again, note that the PSNR values reported in this paper are computed only on the missing entries.}\label{fig:lena:example}
\end{center}
\end{figure*}
The proposed CNNM is indeed a general method for completing partial tensors rather than specific forecasting models. To validate their completion performance, we consider the task of restoring the $200\times200$ Boats image from its incomplete versions. We also include the LRMC method established by~\cite{Candes:2009:math} into comparison.

Figure~\ref{fig:lena:acc} evaluates the recovery performance of various methods. It is clear that LRMC is distinctly outperformed by $\mathrm{DFT}_{\ell_1}$, which is further outperformed largely by CNNM. Figure~\ref{fig:lena:example} shows that the images restored by CNNM is visually better than $\mathrm{DFT}_{\ell_1}$, whose results contain many artifacts but are still better than LRMC. In particular, the second row of Figure~\ref{fig:lena:example} illustrates that CNNM can well handle the ``unidentifiable'' cases where some rows and columns of an image are wholly missing. Table~\ref{tb:lena} shows some detailed evaluation results. As we can see, CNNM works almost equally well under a wide range of kernel sizes. However, since the computational complexity of CNNM is $\mathcal{O}(mk^2)$, the running time grows fast as the enlargement of the kernel size. So, when the computational efficiency is a high priority, it is desirable to use smaller kernels, e.g., $k_1=k_2=13$.
\begin{figure}[h!]
\begin{center}
\includegraphics[width=0.48\textwidth]{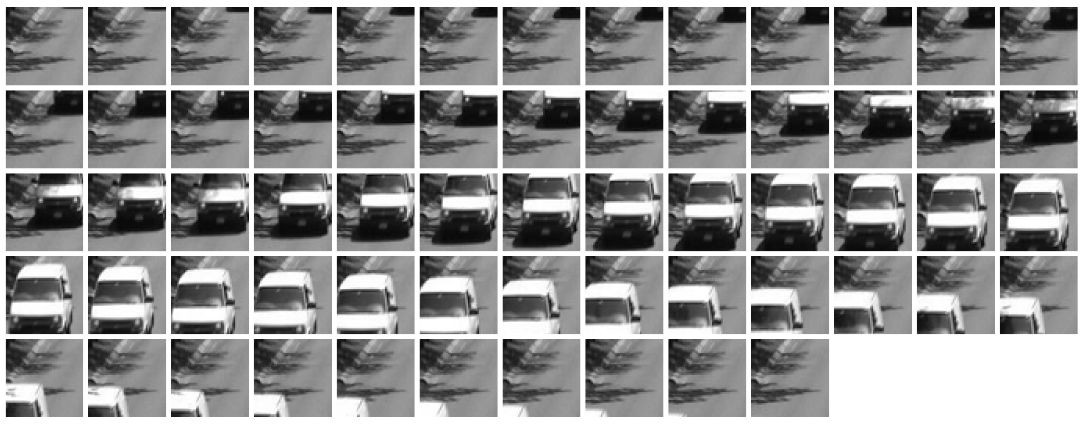}\vspace{-0.1in}
\caption{All 62 frames of the $50\times50\times62$ video used for experiments. This sequence records the entire process that a bus passes through a certain location on a highway. The Gini index of the Fourier transform of this dataset is 0.73.}\label{fig:highway:all}
\end{center}
\end{figure}
\begin{figure}[h!]
\begin{center}
\includegraphics[width=0.4\textwidth]{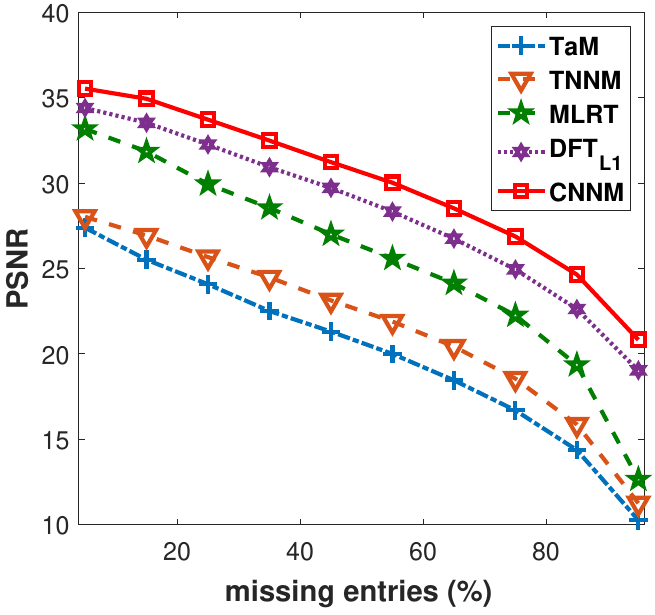}\vspace{-0.15in}
\caption{Evaluation results of restoring the $50\times50\times62$ video from randomly selected entries. The numbers plotted above are averaged from 20 random trials.}\label{fig:highway:acc}
\end{center}
\end{figure}
\subsection{Video Completion and Prediction}
We create a $50\times50\times62$ video consisting of a sequence of $50\times50$ images patches quoted from the CDnet 2014 database~\citep{cdcnet2014}, as demonstrated in Figure~\ref{fig:highway:all}. We first consider a completion task of restoring the video from some randomly chosen entries. To show the advantages of the proposed methods, we include for comparison three tensor completion methods, including Tensor as a Matrix (TaM)~\citep{Candes:2009:math}, Tensor Nuclear Norm Minimization (TNNM)~\citep{tc:tpami:2013:ye} and Mixture of Low-Rank Tensors (MLRT)~\citep{tc:arxiv:2010:ryota}. As shown in Figure~\ref{fig:highway:acc}, $\mathrm{DFT}_{\ell_1}$ dramatically outperforms all competing methods that utilize (Tucker) low-rankness to approximate the structure of videos. But $\mathrm{DFT}_{\ell_1}$ is further outperformed distinctly by CNNM ($13\times13\times13$); this confirms the benefits of controlling the kernel size.
\begin{table}[h!]
\caption{Evaluation results (PSNR) of video prediction. The task is to forecast the last 6 frames given the former 56 frames. For CNNM, the first two quantities of the kernel size are fixed as $k_1=k_2 = 13$.}\label{tb:highway}\vspace{-0.15in}
\begin{center}
\begin{tabular}{ccccccc}\hline
    Methods                      &1st           &2st &3st &4st &5st &6st\\\hline
    TaM                     &5.22	&5.49	&5.75	&5.96	&6.15	&6.23\\
    TNNM                    &5.22	&5.49	&5.75	&5.96	&6.15	&6.23\\
    MLRT                    &5.22	&5.49	&5.75	&5.96	&6.15	&6.23\\\hline
    LSTM                    &20.17  &18.80  &17.01  &16.49  &14.73  &11.83\\
    $\mathrm{DFT}_{\ell_1}$  &22.56	&19.53	&18.72	&19.43	&21.21	&26.20\\
    CNNM($k_3=13$)&23.48&20.64	&20.02	&20.80	&23.07	&28.18\\
    CNNM($k_3=31$)&24.27&21.65	&21.17	&22.10	&24.49	&29.72\\
    CNNM($k_3=62$)&\textbf{24.57}&\textbf{22.01}	&\textbf{21.56}	&\textbf{22.56}	&\textbf{24.98}	&\textbf{30.32}\\\hline
\end{tabular}
\end{center}
\end{table}

We now consider a forecasting task of recovering the last 6 frames from the former 56 frames. Similar to the experimental setup of Section~\ref{sec:exp:uni}, the LSTM network used here also contains 4 layers. The number of input units and LSTM hidden units are 2500 and 500, respectively. Table~\ref{tb:highway} compares the performance of various methods, in terms of PSNR. It can be seen that TaM, TNNM and MLRT produce very poor forecasts. In fact, all the methods built upon Tucker low-rankness may use zero to predict the unseen future entries. As shown in Figure~\ref{fig:highway:res}, $\mathrm{DFT}_{\ell_1}$ owns certain ability to forecast the future data, but the obtained images are full of artifacts. By contrast, the quality of the images predicted by CNNM is much higher.
\begin{figure}[h!]
\begin{center}
\includegraphics[width=0.48\textwidth]{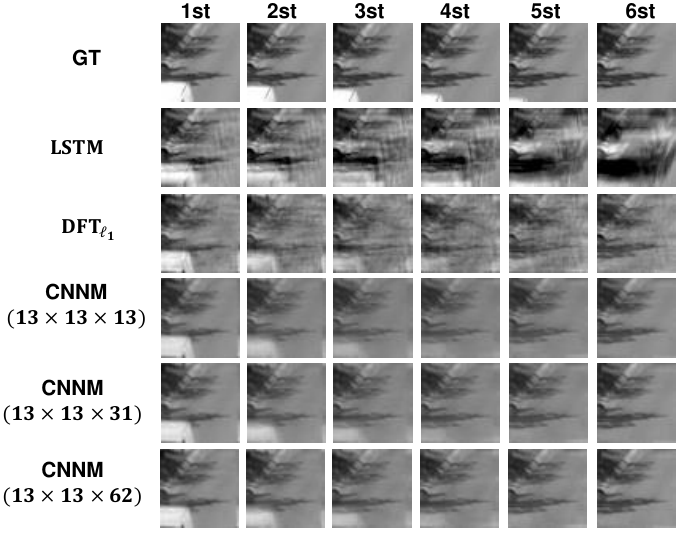}\vspace{-0.15in}
\caption{Visual results of video prediction. The goal of the task is to forecast the last 6 frames given the former 56 frames. GT standards for ``ground truth''.}\label{fig:highway:res}
\end{center}
\end{figure}
\section{Conclusion and Future Work}\label{sec:con}
In this work, we studied the problem of predicting the future values of a time series based on some previously observed data, i.e., the well-known \emph{time series forecasting} (TSF). We first reformulated the TSF problem to an inclusive compressed sensing problem termed \emph{tensor completion with arbitrary sampling} (TCAS), which targets at restoring a tensor from a subset of its entries sampled in an arbitrary fashion. We then showed that the TCAS problem, under certain situations, can be solved by a novel method termed CNNM. Namely, we proved that, whenever the sampling complexity exceeds certain threshold that depends on the convolution rank of the target tensor, CNNM ensures the success of recovery absolutely---or at least to some extent. Experiments on some realistic datasets demonstrated that CNNM is a promising method for data completion and forecasting.

While effective, as one may have noticed, CNNM contributes only to a certain case of the TCAS problem; namely, the data tensor should own particular structures such that its convolution matrix is low-rank or approximately so. There are many kinds of data that may violate such a convolutional low-rankness condition, for which the idea of \emph{representation learning} is quite helpful, as shown in~\cite{lbcnnm:2021}.
\section*{Acknowledgement}
This work is supported in part by New Generation AI Major Project of Ministry of Science and Technology of China under Grant 2018AAA0102501, in part by national Natural Science Foundation of China (NSFC) under Grant U21B2027.
\bibliographystyle{IEEEtran}
\bibliography{ref}
\begin{IEEEbiography}{\textbf{Guangcan Liu}} (Senior Member, IEEE) received the bachelor's degree in mathematics and the Ph.D. degree in computer science and engineering from Shanghai Jiao Tong University, Shanghai, China, in 2004 and 2010, respectively. He was a Post-Doctoral Researcher with the National University of Singapore, Singapore, from 2011 to 2012; the University of Illinois at Urbana¨CChampaign, Champaign, IL, USA, from 2012 to 2013; Cornell University, Ithaca, NY, USA, from 2013 to 2014; and Rutgers University, Piscataway, NJ, USA, in 2014. He was a Professor with the School of Automation, Nanjing University of Information Science and Technology, Nanjing, from 2014 to 2021. He is currently a Professor with the School of Automation, Southeast University, Nanjing, China. His research interests include the areas of machine learning, computer vision, and signal processing.
\end{IEEEbiography}

\begin{IEEEbiography}{\textbf{Wayne Zhang}} received the B.Eng. degree in electronic engineering from the Tsinghua University, Beijing, China, in 2007, the M.Phil. degree in 2009, and Ph.D. degree in 2012, both in information engineering from The Chinese University of Hong Kong. He is currently a Senior Research Director in SenseTime Group Limited. He serves as an EXCO member of AI Specialist Group of Hong Kong Computer Society. His research interests include deep learning and computer vision.
\end{IEEEbiography}

\end{document}